\documentclass[11pt]{article}
\usepackage[bottom=3.5cm, right=2.5cm, left=2.5cm, top=2.5cm]{geometry}

\usepackage[utf8]{inputenc}
\usepackage[T1]{fontenc}
\usepackage{amsmath,amsfonts,amssymb}
\usepackage{amsthm}
\usepackage{hyperref}
\usepackage{url}
\usepackage{booktabs}
\usepackage[english]{babel}
\usepackage{nicefrac}
\usepackage{microtype}
\usepackage{xcolor}
\usepackage{graphicx}
\usepackage{dsfont}
\usepackage{comment}
\usepackage{cancel}
\usepackage{algorithm}
\usepackage{algorithmic}
\usepackage{array}

\theoremstyle{plain}
\newtheorem{theorem}{Theorem}
\newtheorem{lemma}{Lemma}

\newtheorem{cor}{Corollary}

\theoremstyle{definition}

\newtheorem{assumption}{Assumption}

\theoremstyle{remark}
\newtheorem{rmk}{Remark}

\newcommand{\1}{\mathds{1}}
\newcommand{\w}{\boldsymbol{w}}
\newcommand{\x}{\boldsymbol{x}}
\newcommand{\y}{\boldsymbol{y}}
\newcommand{\z}{\boldsymbol{z}}
\newcommand{\e}{\boldsymbol{e}}
\let\v\undefined
\newcommand{\v}{\boldsymbol{v}}
\let\u\undefined
\newcommand{\u}{\boldsymbol{u}}
\let\b\undefined
\newcommand{\b}{\boldsymbol{b}}

\newcommand{\I}{\boldsymbol{I}}

\newcommand{\A}{\boldsymbol{A}}

\newcommand{\Z}{\boldsymbol{Z}}

\newcommand{\R}{\mathbb{R}}
\newcommand{\vnorm}[1]{\|#1\|}
\newcommand{\inner}[1]{\left\langle #1\right\rangle}

\allowdisplaybreaks

\title{Beyond Discreteness: Sample Complexity Analysis of Straight-Through Estimator for 1-bit Quantization}

\author{
Halyun Jeong$^1$ \; Jack Xin$^2$ \; Penghang Yin$^1$\\
$^1$Department of Mathematics and Statistics, University at Albany, SUNY\\
$^2$Department of Mathematics, University of California, Irvine\\
\texttt{\{hjeong2,pyin\}@albany.edu}\\
\texttt{jackx@uci.edu}
}
\date{}

\begin{document}

\maketitle

\begin{abstract}
Training quantized neural networks requires addressing the non-differentiable and discrete nature of the underlying optimization problem. To tackle this challenge, the straight-through estimator (STE) has become the most widely adopted heuristic, allowing backpropagation through discrete operations by introducing biased yet valid surrogate gradients. However, its theoretical properties remain largely unexplored, with few existing analyses focus on the generalization error by assuming an infinite amount of training data. In contrast, this work presents the first sample complexity analysis of STE in the context of neural network quantization. Our theoretical results highlight the critical role of sample size in the success of STE, a key insight absent from existing studies. Specifically, by analyzing the quantization-aware training of a two-layer neural network with binary weights and activations, we derive the sample complexity bounds in terms of the data dimensionality that guarantee the convergence of STE-based optimization to the global minimum for both ergodic and non-ergodic analyses.
Moreover, in the presence of label noises, we prove an intriguing recurrence property of STE-gradient method, where the iterate repeatedly escape from and return to the optimal binary weights. Finally, we empirically demonstrate that STE fails for general non-Gaussian data but its effectiveness can be restored through normalization, underscoring its practical importance in effective quantization.
\end{abstract}

\medskip
\noindent\textbf{Keywords.} quantization, straight-through estimator, 1-bit compressed sensing

\medskip
\noindent\textbf{MSC codes.} 68T07, 68Q32, 90C26, 94A12

\section{Introduction}

Deep neural networks (DNNs) have revolutionized machine learning, achieving remarkable success across various domains, including computer vision \cite{resnet,krizhevsky2017imagenet}, reinforcement learning \cite{mnih2013playing,sutton1999policy}, and natural language processing \cite{vaswani2017attention}. However, state-of-the-art DNNs often contain billions of parameters, requiring substantial computational power and memory. The high resource consumption at inference time poses significant challenges for deploying these models on resource-constrained devices, such as smartphones, embedded systems, and Internet-of-Things (IoT) devices. To address this challenge, network quantization has emerged as a well-adopted and robust
solution, aiming to reduce the precision of weights and activations while preserving model performance. Quantized DNNs operate using low-precision arithmetic rather than standard full-precision floating-point representations, significantly reducing memory footprint and computational cost. This compression technique enables real-time inference on energy-efficient specialized accelerators  \cite{rastegari2016xnor,wang2018training}.

Training quantized neural networks, also known as quantization-aware training (QAT) \cite{ashbrock2021stochastic,halfwave_17,pact,Hubara2017QuantizedNN,rastegari2016xnor,yin2018binaryrelax}, poses significant mathematical challenges due to the discrete nature of low-precision weights and activations, which prohibits the direct use of standard backpropagation. A widely adopted approach to address this challenge is the straight-through estimator (STE) \cite{bengio2013estimating}, which enables gradient-like optimization by using fake gradients for piecewise constant objectives. In essence, STE is an informative heuristic that replaces the almost-everywhere zero derivative of discrete components (e.g., quantized activation function) \emph{exclusively during the backward pass}, yielding an unconventional yet practical "gradient" by modifying the chain rule. STE has played a pivotal role in quantization methods, making it possible to train low-bitwidth neural networks efficiently. Additionally, STE has been extensively used to handle the minimization of discrete-valued objectives that arise in machine learning such as neural
architecture search \cite{riad2022learning,single-path}, discrete latent variables \cite{jang2016categorical,kunes2023gradient,paulus2020rao}, adversarial attack \cite{athalye2018obfuscated}, and sparse recovery \cite{mohamed2023straight}, among others \cite{mao2022enhance,wagstaff2022universal,xu2019relation,yang2022injecting}.
However, despite its empirical success, STE remains an informal heuristic with limited theoretical justification, leading to potential optimization instability and performance degradation.

The main challenge in analyzing STE arises from the absence of a valid gradient of the discrete-valued sample loss $l(\w;\z)$, where $\w$ represents the model parameters and $\z$ is a sample. The theoretical foundations of STE have primarily been studied in the context of neural network quantization \cite{li2017training,yin2018binaryrelax,yin2018understanding}. In particular, \cite{yin2018understanding} was among the first to provide a rigorous analysis of STE for activation quantization. While subsequent studies have explored alternative quantization settings, these analyses focus on minimizing the population loss: $\min_{\w}\,\mathbb{E}_{\z} [l(\w;\z)]$, over the data distribution—an assumption that unrealistically presumes access to infinitely many samples \cite{long2021learning,long2023recurrence,wei2025roste}. Under such an assumption, $\mathbb{E}_{\z} [l(\w;\z)]$ is differentiable  with a valid gradient almost everywhere (a.e.), which enables the analysis of the correlation between the gradient of the population loss $\nabla_{\w} \mathbb{E}_{\z} [l(\w;\z)]$ and the STE gradient $\tilde{\nabla}_{\w} l(\w;\z)$ (also referred to as the coarse gradient \cite{yin2018understanding,yin2018blended}). Here $\tilde{\nabla}_{\w}$ denotes the unconventional STE gradient operator, intended to be distinguished from the standard gradient $\nabla_{\w}$. Notably, \cite{yin2018understanding} provides convergence insights and demonstrates that, when STE is properly chosen, the expectation of STE gradient $\mathbb{E}_{\z} [\tilde{\nabla}_{\w} l(\w;\z)]$, despite being \emph{biased}, maintains a positive correlation with the population gradient $\nabla_{\w} \mathbb{E}_{\z} [l(\w;\z)]$, thereby ensuring a descent property of the STE gradient method for activation-only quantization. With finitely many samples, however, analyzing the empirical risk minimization: $\min_{\w} \frac{1}{N}\sum_{i=1}^N l(\w;\z_i)$, becomes dramatically more challenging, as the empirical loss gradient is a.e. zero. In this setting, a rigorous sample complexity analysis is essential to establish the effectiveness of STE-based optimization.
Moreover, existing works typically consider either weight-only \cite{ajanthan2021mirror,dockhorn2021demystifying,jin2025parq,liu2023binary} or activation-only quantization \cite{long2021learning,yin2018understanding}. The analysis appears to be even more intricate when both weights and activations are quantized, as it involves the use of dual STEs.

\subsection{Contributions}
In this paper, we present the theoretical analysis of the straight-through estimator (STE) for quantization-aware training, addressing an optimization problem that is discrete in both its objective and constraint. Specifically, our main contributions are summarized as follows:
\begin{itemize}
    \item We provide the first analysis of the sample complexity that guarantees the effectiveness of STE. For training a two-layer binarized neural net with random Gaussian inputs, we show that
$O(n^2)$ samples suffice for convergence of the ergodic (averaged) iterates to the optimal solution, while $O(n^4)$ samples are sufficient for the non-ergodic (last iterate) convergence, where $n$ is the dimension of data.
Moreover, we empirically validate that our $O(n^2)$ bound for the ergodic convergence case is tight.

    \item For the non-ergodic convergence analysis, we discover a surprising recurrence behavior of the STE gradient algorithm in the presence of label noise: the iterates repeatedly hit exactly and escape from the optimal binary weights. This stands in contrast to classical problems like linear regression and compressed sensing, where exact recovery of the underlying parameters is typically impossible under observation noise.

    \item We empirically show that convergence guarantees fail to hold for general sub-Gaussian distributions. Nevertheless, we demonstrate that applying a normalization procedure can improve the performance of STE when the data deviate from normality, underscoring the practical importance of normalization in the effective use of STE.

    \item To support the analysis, we leverage techniques from the context of 1-bit compressed sensing \cite{friedlander2021nbiht}, revealing important connections between these two domains. Additionally, we introduce a novel application of occupation time analysis, a tool traditionally used in dynamical systems, to rigorously analyze the drift dynamics of the iterates.

\end{itemize}


\paragraph{Notations.} We clarify the mathematical notations that will
be used throughout this paper: For $m\in\mathbb{N}$, we denote $[m] := \{1,2,\dots,m\}$. We denote vectors by bold
small letters and matrices by bold capital ones.
Moreover, the Hadamard (element-wise) product of two vectors $\x, \y\in\R^n$ is defined as $\x\odot\y := (x_1 y_1, \dots, x_n y_n)\in\R^n$. The Heaviside step function is given by $\1_{\{x\geq 0\}} = 1$ if $x\geq 0$ and $0$ otherwise. In  addition, sign function is defined as
\begin{equation*}
\mathrm{sign}(x)=
 \begin{cases}
    1 & \mbox{if } x\geq 0, \\
    -1 & \mbox{if } x< 0.
 \end{cases}
\end{equation*}
$\forall \, p \in [1,\infty]$, $\|\x\|_p:= (\sum_{i=1}^n |x_i|^p)^{1/p}$ denotes the $\ell_p$ norm of $\x\in\R^n$. In particular, we simply denote by $\|\cdot\|$  the $\ell_2$ norm. If $f$ and $g$ depend on a parameter, say $q$, then we write $f = O(g)$ or $f \lesssim g$ if there exist constant $C > 0$ and $q_0 > 0$ such that $f(q) \le C g(q)$ for all $q > q_0$. Similarly, we write $f = \Omega(g)$ if there exist  constants $C > 0$ and $q_0 > 0$ such that $f(q) \ge C g(q)$ for all $q > q_0$. We write $f \approx g$ if $f = O(g)$ and $f = \Omega(g)$.


\section{Background and Related Work}
\label{sec:background_related_work}

\paragraph{Straight-Through Estimator.} In the classical sense, STE is simply an identity mapping used in the backward pass to replace the a.e. zero derivative of the Heaviside step function $\theta(x):=\1_{\{x\geq 0\}}$. Specifically, this approach treats $\theta^\prime(x)$ as 1, backpropagating as if
$\theta(x)$ were the identity function \cite{hinton2012neural}.
This original idea can be traced back to the celebrated perceptron algorithm \cite{rosenblatt1957perceptron,rosenblatt1962principles} for training single-layer perceptrons, and is now commonly referred to as the identity STE. More generally, for a composite function $l \circ \theta(\x)$, where $l:\R^n \to \R$ is differentiable and $\theta$ is the Heaviside function applied element-wise to the input $\x\in\R^n$, by overloading the notation "$\approx$", the identity STE modifies the chain-rule to evaluate
$$
\nabla_{\x} \, (l\circ \theta)(\x) = \theta^\prime(\x) \odot \nabla \, l(\theta(\x)) \approx \nabla \, l(\theta(\x)),
$$
where $\nabla \, l(\theta(\x))$ denotes the gradient of $l$ evaluated at $\theta(\x)$. Later, variants of the STE heuristic for replacing $\theta^\prime$ have been proposed, including the derivative of sigmoid function \cite{bengio2013estimating,bnn_16}. In an orthogonal direction, the expressiveness and approximation power of neural nets with binary weights or activations were investigated in \cite{gunturk2021approximation,shen2021neural,siegel2022high,vershynin2020memory}. More recently, the use of STE has been extended to facilitate backpropagation through multi-bit quantized (multi-step) functions in the quantization of deep neural networks, where more sophisticated STEs such as the derivatives of ReLU, log-tailed ReLU, and clipped ReLU have been adopted \cite{halfwave_17,pact,Hubara2017QuantizedNN,kim2020binaryduo,yin2018blended}.

Theoretical underpinnings of STE were initially developed in the context of weight-only quantization, focusing on minimizing general objective functions with quantized parameters using the identity STE \cite{li2017training,yin2018binaryrelax}. A key advancement in understanding why STE works was made in \cite{yin2018understanding}, which analyzed the expected neural network loss over Gaussian inputs with binary activations, enabling a fine-grained study of more sophisticated STEs beyond the identity STE. Notably, \cite{long2021learning, yin2018understanding} also demonstrated that effective STE choices for activation quantization are typically not unique. However, all prior works either consider general loss functions \cite{ajanthan2021mirror,jin2025parq,li2017training,liu2023binary,yin2018binaryrelax} or population network losses \cite{long2021learning,long2023recurrence,wei2025roste,yin2018understanding}, in contrast to the finite-sample setting addressed in this work.



\paragraph{1-bit Compressed Sensing.}
The 1-bit compressed sensing problem was first introduced by \cite{boufounos20081}. The objective is to recover an $s$-sparse signal $\x$ from its 1-bit measurements, given by $\b = \mathrm{sign}(\A\x)$, which are the signs of the linear measurements, where $\A$ is typically a random matrix such as Gaussian. While 1-bit compressed sensing offers a practical approach to discretizing measurements for efficient sparse recovery, due to this discontinuous nature, it has been posing several interesting challenges.
Numerous approaches have been proposed to address the 1-bit compressed sensing problem, including linear programming methods \cite{plan2013one}, convex relaxations \cite{plan2012robust}, and the generalized Lasso \cite{plan2016generalized}, among others.  However, these methods provide theoretical guarantees with only sub-optimal reconstruction error at most $O(\sqrt{s/N})$, where $N$ is the number of samples. \cite{jacques2013robust} proposed Normalized Binary Iterative Hard Thresholding (NBIHT), and it has been known for its strong empirical performance for 1-bit compressed sensing. Despite the analytical difficulties arising from the discrete and non-smooth loss induced by quantization, recent breakthroughs have established NBIHT's near-optimal sample complexity. \cite{friedlander2021nbiht} introduced Restricted Approximate Invertibility Condition (RAIC), resembling the Restricted Isometry Property (RIP) for conventional compressed sensing  \cite{baraniuk2008simple,candes2006robust,eldar2012compressed,foucart2013invitation}. Subsequent works \cite{chen2024optimal,matsumoto2024binary} have further affirmed its optimality concerning the sparsity level $s$, thereby showing the optimal reconstruction error of $O(s\log N/N)$, matching the lower bound of reconstruction error found in \cite{jacques2013robust}.

\paragraph{Dynamical Systems and Ergodic Theory.}
The argument we used to establish our ergodic theorem, which is based on occupation time analysis, shares a similar flavor with Kac's Lemma or Kakutani's skyscraper argument \cite{dani2012ergodic, petersen1989ergodic} in Ergodic Theory.
However, since our setting lacks the formal measure-theoretic structure typically assumed in ergodic theory (e.g., invariant measures or measure-preserving transformations), our analysis was developed independently. Nonetheless, its emphasis on recurrence through occupation time is reminiscent of the spirit of ergodic-theoretic arguments.

\section{Problem Setup}
Given a sample $\Z \in \R^{m \times n}$, we consider the following two-layer neural network as in \cite{long2021learning}, where $\w \in \R^n$ represents trainable weights, while $\v \in \R^m\setminus \{\mathbf{0}\}$ is fixed and known:
\begin{equation*}
    y(\w;\Z) := \sum_{j=1}^m v_j \theta(\z_j^\top \w) = \v^\top \theta(\Z \w) \in\R.
\end{equation*}
Here, $\theta(x) = \1_{\{x \geq 0\}}$ denotes the binary (Heaviside) activation function, and $\z_j^\top$ represents the $j$-th row of the data matrix $\Z\in \R^{m \times n}$. Each $\z_j\in\R^n$ can be interpreted as an image patch in computer vision (CV) tasks or as a token embedding in natural language processing (NLP) applications. Accordingly, $n$ represents the filter size of a convolutional layer in CV (see Figure~\ref{fig:linear_layer}) or embedding dimension in NLP, and  $m$ corresponds to the number of image patches or the context length of a sample, respectively. The last (or second) layer is a fully-connected layer, which is kept unquantized in practice. Hence, the loss remains differentiable w.r.t. the weights of the second layer, thus not requiring the use of STE for its training. Since our work aims to analyze the convergence behavior of the STE, the quantization of the second layer is of independent interest.

\begin{figure}[h!]
    \centering
    \includegraphics[width=0.5\linewidth]{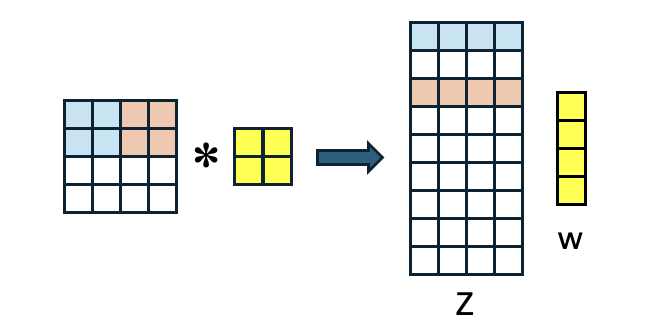}
    \caption{The first layer can be interpreted as an unrolled single-channel convolutional layer, where $*$ denotes the convolution operation.}
    \label{fig:linear_layer}
\end{figure}

\begin{assumption}[Input Data]\label{assump:data}
The row vectors of input data, $\z_j\in\R^n$ for all $j\in[m]$, follow i.i.d. standard Gaussian distribution $\mathcal{N}(0,\I_n)$. We also make the mild assumption that $m\lesssim O(2^n)$.
\end{assumption}

We assume the label $y_{\Z}$ is generated by an underlying model with binary weights $\w^*$, where the second-layer weights $\v\neq \mathbf{0}$ is given. Note that since $y$ is positively scale-invariant to $\w$, meaning that $y(c \, \w;\Z) = y(\w;\Z)$ for all $c>0$, we assume for simplicity that the true weights $\w^*\in\mathbb{Q}_1 := \left\{\pm \frac{1}{\sqrt{n}}\right\}^n$ is unit-normed. That said, the label $y_{\Z}$ is generated by
    \begin{equation}\label{eq:labels}
    y_{\Z} := y(\w^*;\Z) + \xi = \v^\top \theta(\Z \w^*) + \xi,
\end{equation}
which is possibly corrupted with unknown additive noise $\xi\in\R$.

\begin{assumption}[Label Noise]\label{assump:noise}
 The label noise $\xi$ follows symmetric and sub-Gaussian distribution with its sub-Gaussian norm bounded by $K_\xi$.
 In particular, for mean-zero Gaussian noise with variance $\sigma^2$, we have $K_\xi \approx \sigma$.
\end{assumption}

Moreover, we consider the following squared loss function:
\begin{equation*}\label{sample_loss}
\begin{aligned}
l(\w; \Z):=\frac{1}{2}\left(y(\w; \Z)- y_{\Z} \right)^2
= \frac{1}{2}\left(\v^\top\theta(\Z\w)- y_{\Z} \right)^2.
\end{aligned}
\end{equation*}

Then given the labeled dataset $\{(\Z^{(i)}, y_{\Z^{(i)}})\}_{i=1}^N$ and sample loss 

$l_i(\w): = l(\w;\Z^{(i)})$, training fully binarized model requires to solve the following empirical risk minimization problem:
\begin{equation}\label{eq:qat}
\min_{\w} L(\w):= \frac{1}{N}\sum_{i=1}^N l_i(\w) = \frac{1}{2N}\sum_{i=1}^N  \left(\v^\top\theta(\Z^{(i)}\w)- y_{\Z^{(i)}} \right)^2 \quad \mbox{subject to} \quad  \w\in\mathbb{Q}_1.
\end{equation}
Note that the above optimization problem is \emph{discrete, both in its objective and constraint}. We define the weight quantizer $\mathcal{Q}$, which is the projection of $\w$ onto $\mathbb{Q}_1= \left\{\pm \frac{1}{\sqrt{n}}\right\}^n$, as
\begin{equation}\label{eq:quantizer}
\mathcal{Q}(\w):= \mathrm{proj}_{\mathbb{Q}_1}(\w)  = \frac{1}{\sqrt{n}}\mathrm{sign}(\w).
\end{equation}

\subsection{STE-Gradient Method}
Since the activation function $\theta$ is piece-wise constant, the sample loss $l$ has a vanished gradient:
$$
\nabla_{\w} l = \Z^\top (\theta^\prime(\Z \w)\odot \v) \left(\v^\top\theta(\Z\w)- y_{\Z} \right) = \mathbf{0} \quad \mbox{a.e.},
$$
which renders the standard backpropagation inapplicable.
To circumvent this issue, we consider a surrogate gradient obtained by replacing $\theta^\prime$ with the derivative of ReLU $\mu(x) := \max\{x,0\}$, referred to as ReLU STE:
\begin{equation}\label{eq:ste_sample}
    \tilde{\nabla}_{\w} l = \Z^\top (\mu^\prime(\Z \w)\odot \v) \left(\v^\top\theta(\Z\w)- y_{\Z} \right),
\end{equation}
where $\mu^\prime(x) = \1_{\{x\geq 0\}}$. ReLU STE was proved effective for activation quantization of shallow neural nets \cite{yin2018understanding} such as LeNet-5 \cite{mnist_98}.

Then we have
\begin{equation}\label{eq:ste}
\tilde{\nabla}_{\w} L(\w) = \frac{1}{N}\sum_{i=1}^N \tilde{\nabla}_{\w} l_i(\w) =  \frac{1}{N}\sum_{i=1}^N (\Z^{(i)})^\top (\mu^\prime(\Z^{(i)} \w)\odot \v) \left(\v^\top\theta(\Z^{(i)}\w)- y_{\Z^{(i)}} \right).
\end{equation}
In this paper, we study the most widely used quantization-aware training (QAT) algorithm \cite{halfwave_17,Hubara2017QuantizedNN,yin2018blended}, which we refer to as STE-gradient method. It updates the weights using STE-induced surrogate gradient while keeping them quantized:
\begin{equation}\label{eq:ste_gradient}
    \begin{cases}
     \x^t = \x^{t-1} - \eta_t\tilde{\nabla}_{\w} L(\w^{t-1})\\
     \w^t = \mathcal{Q}(\x^t).
    \end{cases}
\end{equation}
The STE-gradient method (\ref{eq:ste_gradient}) resembles the classical projected gradient descent (PGD) iteration $\w^t = \mathcal{Q}(\w^{t-1}- \eta_t \nabla_{\w} L(\w^{t-1}))$,
but features two key modifications that make it particularly effective for the discrete optimization problem (\ref{eq:qat}). First, it employs the STE-gradient to bypass the zero gradient issue arising from quantized activation $\theta$. Second, unlike PGD, the gradient step of STE-gradient method combines the continuous variable $\x$ and the (STE-)gradient evaluated at $\w$,
allowing informative gradients to aggregate over time at small step size $\eta_t$ without being offset by the discretization $\mathcal{Q}$.

Hereby we summarize the STE-gradient method for solving (\ref{eq:qat}) in Algorithm \ref{alg:ste}, assuming $\eta_t$ satisfies the condition below.

\begin{assumption}[Step Size]
\label{assump:step_size}
The step size $\eta_t$ satisfies: (1) Boundedness: $\forall t$, $0<\eta_{t+1} \le \eta_t \le \eta_{\max} < \infty$; (2) Non-Summability: $\sum_{t=1}^\infty \eta_t = \infty$; (3) Slow Decay: $\lim_{t\to\infty} \frac{\eta_{t+1}}{\eta_t} = 1$.
In particular, $\eta_t$ can be the constant step size or diminishing step size with $\eta_t \propto t^{-p}$ for $p \in (0, 1]$.
\end{assumption}


In fact, (\ref{eq:ste_gradient}) is equivalent to the following single-step iteration:
\begin{equation}\label{eq:ste_single}
    \x^t = \x^{t-1} - \eta_t\tilde{\nabla}_{\w} L(\mathcal{Q}(\x^{t-1})).
\end{equation}
Here, $\x\in\R^n$ represents the latent continuous variables that are quantized to $\w\in\mathbb{Q}_1$, i.e., $\w = \mathcal{Q}(\x)$, and the STE gradient is evaluated for $L$ at $\w = \mathcal{Q}(\x)$. The above iteration essentially addresses the following unconstrained minimization that is equivalent to problem (\ref{eq:qat}):
$$
\min_{\x\in\R^n} L(\mathcal{Q}(\x)),
$$
using $\emph{dual STEs}$ to construct the surrogate of
$$
\nabla_{\x} L(\mathcal{Q}(\x)) = \mathcal{Q}^{\prime}(\x) \odot \nabla_{\w} L(\mathcal{Q}(\x)) \approx \tilde{\nabla}_{\w} L(\mathcal{Q}(\x)),
$$
where the ReLU STE addresses the discreteness of $L$, and the identity STE handles that of the quantizer $\mathcal{Q}$.

\begin{algorithm}[h]
\caption{STE-based Training of Binary Neural Network.}
\label{alg:ste}
\begin{algorithmic}
\STATE {\bfseries Input:}  max iteration number $T$, step size $\eta_t>0$.
\STATE {\bfseries Initialize:} latent float weights $\x^0\in\R^n$, binary weights $\w^0 = \mathcal{Q}(\x^0)$.
\FOR{$t=1$ {\bfseries to} $T$}
\STATE $\x^t = \x^{t-1} - \eta_t\tilde{\nabla}_{\w} L(\w^{t-1})$ \hfill $\triangleright$ STE-gradient $\tilde{\nabla}_{\w} L$ is given by (\ref{eq:ste})
\STATE  $\w^t = \mathcal{Q}(\x^t)$ \hfill $\triangleright$ quantize the weights by (\ref{eq:quantizer})
\ENDFOR
\end{algorithmic}
\end{algorithm}

\section{Main Results}\label{sec:main_results}
We present the first finite-sample analysis of STE for 

quantization-aware training. Our main results provide global convergence guarantees for Algorithm~\ref{alg:ste}, assuming a sufficiently large sample size $N$. Specifically, we show that $O(n^2)$ samples are sufficient to ensure ergodic convergence to $\w^*$, while $O(n^4)$ samples suffice for non-ergodic convergence, where the iterates $\{\w^t\}$ visit $\w^*$ infinitely often.




\subsection{Ergodic Convergence}
We demonstrate that the ergodic average of the iterates of the STE-gradient method, $\overline{\w}^T: = \frac{1}{T}\sum_{t=1}^T \w^t$,  converges to $\w^*$ up to an error floor that depends only on the ratio $n/N$.

\begin{theorem}[Informal] Under Assumptions~\ref{assump:data}--\ref{assump:step_size},
\label{thm:Informal_ergodic_informal_nonasymp}
when the sample size
$$
N \gtrsim  \left(\tfrac{(C'\|\v\|_1^2 +  C''K_\xi\|\v\|_1)n}{\,\,\|\v\|^2}\right)^2,
$$
for some universal constants $C'$ and $C''$, then with high probability,
 the ergodic average of iterates $\overline{\w}^T$ generated by Algorithm \ref{alg:ste} satisfies for a sufficiently large $T$ that
\[
\left\| \overline{\w}^T - \w^* \right\|_\infty \lesssim  O\left(\sqrt{\frac{n}{N}} + \frac{1}{\sqrt{n} T}\right).
\]
Moreover, the exact recovery can be achieved by
$\mathcal{Q}\left( \overline{\w}^T \right) =\w^*$
for sufficiently large $T$.
\end{theorem}



The sample complexity bound depends on $K_\xi$, the sub-Gaussian norm of noise $\xi$; heavier noise requires more samples to guarantee convergence.
Moreover, it also depends on $\|\v\|_1/\|\v\|_2$, which can be interpreted as the effective sparsity of $\v$ \cite{yin2014ratio}.
 In low-noise regime where $K_\xi  \ll \|\v\|_1$, the effective sparsity primarily determines the sample complexity $N$. In high-noise regime where $K_\xi \gg \|\v\|_1$, the noise level $K_\xi$ becomes the dominant factor. This trade-off similarly applies to the non-ergodic convergence case.

\subsection{Non-Ergodic Convergence}
\begin{theorem}[Informal]
\label{thm:global_coordinate_recovery_noisy}
Under Assumptions~\ref{assump:data}--\ref{assump:step_size}, there exist universal constants $C'$ and $C''$ such that, if the sample size satisfies $N\gtrsim\left(\tfrac{(C'\|\v\|_1^2 +  C''K_\xi\|\v\|_1) n^2}{\,\,\|\v\|^2}\right)^2$, the iterates $\{\w^t\}$ generated by Algorithm \ref{alg:ste} visit $\w^*$ infinitely often with high probability.
Furthermore, if the noise $\xi$ is non-zero and continuous, $\{\w^t\}$ also depart from $\w^*$ infinitely often with probability roughly at least $1/2$ over the randomness of sample draw $\{(\Z^{(i)}, \xi^{(i)})\}_{i=1}^N$.
\end{theorem}
We remark that this departure-probability lower bound is conservative, and the true escape probability could be significantly higher in practice.

To understand the instability behavior of iterates $\{\w^t\}$ around $\w^*$ in the presence of label noise, we recall that STE-gradient method is equivalent to the single-step iteration in (\ref{eq:ste_single}):
$$
\x^{t} = \x^{t-1} - \eta_t\tilde{\nabla}_{\w} L(\mathcal{Q}(\x^{t-1})).
$$
When the optimal solution is achieved at $\w^{T} = \mathcal{Q}(\x^{T})= \w^*$ for some $T$, the corresponding STE-gradient with $\mu^\prime(x) = \1_{\{x\geq 0\}}$ is given by
$$
\tilde{\nabla}_{\w} L(\mathcal{Q}(\x^{T})) =  \tilde{\nabla}_{\w} L(\w^*)= -\frac{1}{N}\sum_{i=1}^N (\Z^{(i)})^\top (\mu'(\Z^{(i)}\w^*)\odot\v) \xi^{(i)}
$$
is nonzero with probability $1-2^{-\|v\|_0N}$ under the continuous nonzero-noise assumptions, where $\|v\|_0:=|\operatorname{supp}(\v)|$. Unless $\mathrm{sign}(\w^*) = \mathrm{sign}(-\tilde{\nabla}_{\w} L(\w^*))$, $\{\x^{t}\}_{t> T}$ will move in a fixed direction until $\w^t = Q(\x^t)\neq \w^*$ for some $t > T$. Subsequently, $\w^t$ will repeatedly revisit and escape from $\w^*$. Similar oscillation behavior was also observed empirically in ImageNet classification towards the end of quantization-aware training using STE \cite{nagel2022overcoming}.

However, we emphasize that this divergence behavior is \emph{beneficial}, as it helps prevent the iterates from stalling. In contrast, even with the STE-gradient $\tilde{\nabla}_{\w} L$, any $\w\in\mathbb{Q}_1$ becomes a stationary point of the following PGD-like iteration under a small $\eta_t>0$:
$$
\w^t = \mathcal{Q}(\w^{t-1}- \eta_t \tilde{\nabla}_{\w} L(\w^{t-1})),
$$
since $\mathcal{Q}$ is a discrete operator. This will lead to stagnation and thus poor empirical performance.

\section{Proof Ideas}
This section outlines the proofs of Theorems~\ref{thm:Informal_ergodic_informal_nonasymp} and~\ref{thm:global_coordinate_recovery_noisy}. The first ingredient is a finite-sample concentration estimate showing that the STE-gradient is uniformly close to a linear drift toward the planted binary vector $\w^*$. The second ingredient is a coordinate-wise occupation-time analysis for the latent dynamics generated by Algorithm~\ref{alg:ste}.

\subsection{$\ell_\infty$ Concentration Bound for the STE-Gradient}\label{sec:concentration}
We first record the sample-size condition used throughout the concentration argument.

\begin{assumption}\label{assump:sample_size}
There exists some sufficiently large universal constant $C_1 > 0$ such that $N \ge {C_1 n^2}$.
\end{assumption}

The following theorem is the finite-sample analogue of the population STE alignment property. It shows that, uniformly over the quantized parameter space $\mathbb Q_1$, the empirical STE-gradient behaves like the linear map
\[
\w \mapsto \frac{\|\v\|^2}{\tau}(\w-\w^*),
\qquad \tau=2\sqrt{2\pi},
\]
up to an $\ell_\infty$ perturbation of order $\sqrt{n/N}$. This is the restricted approximate invertibility principle that drives the convergence analysis.

\begin{theorem}
\label{thm:RAIC}
\label{thm:RAIC_for_two_layer_quantized_neural_networks_appendix}
Suppose that Assumptions~\ref{assump:data}, \ref{assump:noise}, and~\ref{assump:sample_size} hold. Let $\w^* \in \mathbb{Q}_1$  and $\tau = 2\sqrt{2\pi}$. Then, there exist positive universal constants $C_1, C', C'', c$ such that  with probability at least $1 - 5m\exp(-cn)$,
\begin{align*}
\left \|{\|\v\|^2 \over \tau } (\w-\w^*)  -  \tilde{\nabla}_{\w} L(\w) \right \|_\infty \le {1 \over \tau} (C' \|\v\|_1^2 + C''K_\xi \|\v\|_1) \sqrt{ \frac{n}{N} }
\end{align*}
for all $\w \in \mathbb{Q}_1$.
\end{theorem}

To see the mechanism behind Theorem~\ref{thm:RAIC}, we decompose the STE-gradient into a noiseless term and a noise term:
\begin{align*}
\tilde{\nabla}_{\w} L(\w)
= & \underbrace{\frac{1}{N}\sum_{i=1}^N (\Z^{(i)})^\top (\mu^\prime(\Z^{(i)} \w)\odot \v) \v^\top\left(\theta(\Z^{(i)}\w)- \theta(\Z^{(i)}\w^*) \right)}_{\text{noiseless term}} \\
& \qquad \qquad \qquad \qquad \qquad \qquad \qquad - \underbrace{\frac{1}{N}\sum_{i=1}^N (\Z^{(i)})^\top \Bigl[\mu'(\Z^{(i)}\,\w)\odot\v\Bigr]\, \xi^{(i)}}_{\text{noise term}}.
\end{align*}
The signal is carried by the diagonal part of the noiseless term, where the same row $\z_j^{(i)}$ appears both in the STE factor and in the activation mismatch. The off-diagonal part is a cross term between different rows and is controlled separately. Using $\mu^\prime(x)=\theta(x)=\1_{\{x>0\}}$, we write
\small
\begin{align*}
    \mathrm{noiseless \; term} &= \underbrace{ \frac{1}{N}\sum_{i=1}^N\sum_{j=1}^m
   v_j^2\,
   \Bigl(\1_{\bigl\{\z_j^{(i)\top} \w>0\bigr\}}\,\z_j^{(i)}\Bigr)\;
   \Bigl(\1_{\{\z_j^{(i)\top} \w>0\}} - \1_{\{\z_j^{(i)\top} \w^*>0\}}\Bigr)}_{\text{squared term}}\\
& \qquad + \underbrace{\frac{1}{N}
\sum_{i=1}^N\sum_{\substack{j,k =1\\j\neq k}}^m
   v_j\,v_k\,
   \Bigl(\1_{\bigl\{\z_j^{(i)\top} \w>0\bigr\}}\,\z_j^{(i)}\Bigr)\;
   \Bigl(\1_{\{\z_k^{(i)\top} \w>0\}} - \1_{\{\z_k^{(i)\top} \w^*>0\}}\Bigr)}_{\text{cross term}}.
\end{align*}
\normalsize
Among these three pieces, the squared term is the informative one: it is the part that produces the linear drift toward $\w^*$. The next single-row estimate isolates this mechanism and shows why the STE-gradient behaves like a restricted approximate inverse on the quantized set $\mathbb Q_1$.

For the squared term, define
\begin{align*}
\widetilde{g}^{(j)}(\w)
&= \frac{1}{N}
\sum_{i=1}^N
   \Bigl(\1_{\{\z_j^{(i)\top} \w>0\}} - \1_{\{\z_j^{(i)\top} \w^*>0\}}\Bigr)  \Bigl(\1_{\bigl\{\z_j^{(i)\top} \w>0\bigr\}}\,\z_j^{(i)}\Bigr)\;,
\end{align*}
so that the squared term is given by $\sum_{j=1}^m
   v_j^2\,\widetilde{g}^{(j)}(\w)$. To make the notation light, define
\[
D(\w;\z^{(i)})
  \;:=\; \1_{\{\z^{(i)\top} \w>0\}}
          - \1_{\{\z^{(i)\top} \w^*>0\}}
\]  and
\[
\widetilde{g}(\w) =
\frac{1}{N}\sum_{i=1}^N D(\w;\z^{(i)}) \1_{\bigl\{\z^{(i)\top} \w>0\bigr\}}\,\z^{(i)}.
\]

\begin{theorem}
\label{thm:RAIC_single_term}
Under the same conditions in Theorem~\ref{thm:RAIC_for_two_layer_quantized_neural_networks_appendix}, there exist some positive universal constants $c_7$ and $C_4$ such that, with probability at least $1 - 3 \exp(-c_7 n) $, we have
\[
\|\w-\w^* - \tau \cdot \widetilde{g}(\w)\|_\infty \le C_4\sqrt{\frac{n}{N}}.
\]
\end{theorem}

The following corollary of Theorem~\ref{thm:RAIC_single_term} gives the concentration bound of the squared term:
\begin{cor}
\label{cor:noiseless_part_bound}
Under the same conditions in Theorem~\ref{thm:RAIC_for_two_layer_quantized_neural_networks_appendix}, there exist some positive universal constants $c_7$ and $C_4$ such that, with probability at least $1 - 3m \exp(-c_7 n)$,
\begin{align*}
\nonumber
&\left \|\|\v\|^2(\w-\w^*)  - \tau  \sum_{j=1}^m
   v_j^2\,\widetilde{g}^{(j)}(\w) \right \|_\infty \le C_4 \|\v\|_2^2 \sqrt{\frac{n}{N}}
\end{align*}
\normalsize
for all $\w \in \mathbb{Q}_1$.
\end{cor}
\begin{proof}[Proof of Corollary~\ref{cor:noiseless_part_bound}.]
By the triangle inequality,
\begin{align*}
\nonumber
&\left \|\|\v\|^2(\w-\w^*)  - \tau  \sum_{j=1}^m
v_j^2\,\widetilde{g}^{(j)}(\w) \right \|_\infty
\le \sum_{j=1}^m v_j^2 \left \|(\w-\w^*)  - \tau
\widetilde{g}^{(j)}(\w) \right \|_\infty \\
&\le \sum_{j=1}^m v_j^2  C_4 \sqrt{\frac{n}{N}}= C_4 \|\v\|^2 \sqrt{\frac{n}{N}},
\end{align*}
where the second inequality is by applying Theorem~\ref{thm:RAIC_single_term} to each $\widetilde{g}^{(j)}$ for $j \in [m]$ and by the union bound of the probabilities.
\end{proof}

\begin{proof}[Proof of Theorem~\ref{thm:RAIC_single_term}.]
Our proof strategy for Theorem~\ref{thm:RAIC_single_term} adapts techniques from the NBIHT analysis for 1-bit compressed sensing \cite{chen2024optimal,friedlander2021nbiht, matsumoto2024binary}.
In particular, for $\w \in \mathbb{Q}_1\setminus \{\pm \w^*\}$, we consider the following orthogonal decomposition of $\z^{(i)}$:
\begin{equation*}
\z^{(i)} = \inner{\z^{(i)},  {\w - \w^* \over \|\w - \w^*\|} }  {\w - \w^* \over \|\w - \w^*\|} + \inner{\z^{(i)},  {\w + \w^* \over \|\w + \w^*\|} }  {\w + \w^* \over \|\w + \w^*\|} + \b_i(\w),
\end{equation*}
where $\b_i(\w)$ is orthogonal to $\w - \w^*$ and $\w + \w^*$. Note that $\w - \w^*$ and $\w + \w^*$ are also orthogonal since $\|\w\| = \|\w^*\| = 1$.

Using this three-component decomposition of $\z^{(i)}$, we then establish the bound as follows:

\begin{align*}
\nonumber
&\|\w - \w^*- \tau \cdot \widetilde{g}(\w)\|_\infty \\
&= \left \|\w - \w^* - {\tau  \over N} \sum_{i=1}^N \Bigl(D(\w;\z^{(i)})\Bigr)\Bigl(\1_{\bigl\{\z^{(i)\top} \w>0\bigr\}}\,\z^{(i)}\Bigr) \right \|_\infty\\
&\le \left \|\w - \w^* - {\tau \over N}  \sum_{i=1}^N \Bigl(D(\w;\z^{(i)})\Bigr)  \inner{\1_{\bigl\{\z^{(i)\top} \w>0\bigr\}}\,\z^{(i)},  {\w - \w^* \over \|\w - \w^*\|} }  {\w - \w^* \over \|\w - \w^*\|} \right \|_\infty \\
&\quad + \left \|  {\tau \over N}   \sum_{i=1}^N \Bigl(D(\w;\z^{(i)})\Bigr)  \inner{\1_{\bigl\{\z^{(i)\top} \w>0\bigr\}}\,\z^{(i)},  {\w + \w^* \over \|\w + \w^*\|} }  {\w + \w^* \over \|\w + \w^*\|} \right \|_\infty \\
&\quad + \left \|  {\tau \over N}    \sum_{i=1}^N \Bigl(D(\w;\z^{(i)})\Bigr) \1_{\bigl\{\z^{(i)\top} \w>0\bigr\}}  \b_i(\w)  \right \|_\infty \\
&\le  \left | \|\w - \w^*\| - {\tau \over N}  \sum_{i=1}^N \Bigl(D(\w;\z^{(i)})\Bigr) \1_{\bigl\{\z^{(i)\top} \w>0\bigr\}}  \inner{\z^{(i)},  {\w - \w^* \over \|\w - \w^*\|} } \right | \\
&\quad +  \left |   {\tau \over N}   \sum_{i=1}^N \Bigl(D(\w;\z^{(i)})\Bigr) \1_{\bigl\{\z^{(i)\top} \w>0\bigr\}} \inner{ \z^{(i)},  {\w + \w^* \over \|\w + \w^*\|} }  \right | \\
&\quad + \left \|{\tau \over N}    \sum_{i=1}^N \Bigl(D(\w;\z^{(i)})\Bigr) \1_{\bigl\{\z^{(i)\top} \w>0\bigr\}}  \b_i(\w) \right \|_\infty\\
&\le C_3 \tau \sqrt{n \over N} + C_3' \tau \sqrt{n \over N} + C_3'''\,\tau\;\sqrt{\frac{n}{N}}\\
&\le C_4\sqrt{\frac{n}{N}}
\end{align*}
where we have applied Lemmas \ref{lem:Concentration_of_main_term_1}, \ref{lem:Concentration_of_main_term_2}, and \ref{lem:Concentration_of_main_term_3} followed by the triangle inequality. 
The common idea behind these three lemmas is to fix
\(\w\in\mathbb Q_1\), identify the population mean of each projected component,
and then apply sub-Gaussian concentration followed by a union bound over the
finite quantized set \(\mathbb Q_1\). The first projected component
recovers the population drift in the \(\w-\w^*\) direction, with magnitude of
order \(\|\w-\w^*\|\). The second and
residual components are centered,  by the orthogonality of
\(\w-\w^*\) and \(\w+\w^*\) since $\|\w\|=\|\w^*\|=1$, and by projection onto the orthogonal complement.
Consequently, the deviations of all three components from their population
contributions are controlled uniformly, for the relevant
\(\w\in\mathbb Q_1\setminus\{\pm\w^*\}\), at order \(O(\sqrt{n/N})\). For more details, see the proofs of Lemmas \ref{lem:Concentration_of_main_term_1}, \ref{lem:Concentration_of_main_term_2}, and \ref{lem:Concentration_of_main_term_3} in the Appendix. 

It remains to combine the success probabilities of the three estimates. 

The success probability follows from the union bound by taking the minimum of the positive constants $c_2, c_2', c_2'''$ in the lemmas such that
$1 - \exp(-c_2 n) - \exp(-c_2' n) - \exp(-c_2''' n) \ge 1 - 3\exp(-c_7 n)$, where $c_7$ is the minimum.

For the case $\w = -\w^*$, the same argument holds but with a simpler decomposition of $\z^{(i)}$.
\begin{align*}
\z^{(i)} = \inner{\z^{(i)},  {\w - \w^* \over \|\w - \w^*\|} }  {\w - \w^* \over \|\w - \w^*\|}  + \widetilde{\b_i}(\w),
\end{align*} where $\widetilde{\b_i}(\w)$ is orthogonal to $\w - \w^* = -2\w^*$. Applying the Lemmas \ref{lem:Concentration_of_main_term_1} and \ref{lem:Concentration_of_main_term_3} as before yields
\begin{align*}
\nonumber
&\|\w - \w^*- \tau \cdot \widetilde{g}(\w)\|_\infty \\
&\le  \left | \|\w - \w^*\|_2 - {\tau \over N}  \sum_{i=1}^N D(\w;\z^{(i)}) \1_{\bigl\{\z^{(i)\top} \w>0\bigr\}}  \inner{\z^{(i)},  {\w - \w^* \over \|\w - \w^*\|} } \right | \\
&\quad + \left \|{\tau \over N}    \sum_{i=1}^N D(\w;\z^{(i)}) \1_{\bigl\{\z^{(i)\top} \w>0\bigr\}} \widetilde{\b_i}(\w) \right \|_\infty\\
&\le C_3 \tau \sqrt{n \over N} + C_3'''\,\tau\;\sqrt{\frac{n}{N}} \le C_4\sqrt{\frac{n}{N}}.
\end{align*}
In the case when $\w = \w^*$,  the bound in Theorem~\ref{thm:RAIC_single_term} holds trivially.

\end{proof}

The remaining pieces are centered. The cross term is controlled by independence between different Gaussian rows, and the noise term is controlled by the sub-Gaussian noise assumption. Their estimates are stated and proved in Appendix~\ref{sec:corss_term} and Appendix~\ref{sec:noise_term}. Combining these estimates with Corollary~\ref{cor:noiseless_part_bound} gives the full STE-gradient concentration theorem.

\begin{proof}[Proof of Theorem~\ref{thm:RAIC_for_two_layer_quantized_neural_networks_appendix}.]
Finally, combining the bound for the squared term part in Corollary~\ref{cor:noiseless_part_bound} and the cross term part in Lemma~\ref{lem:Concentration_of_cross_term_linf_alt}, and the noisy part in Lemma~\ref{lem:Concentration_of_noise_term_uniform} and applying the triangle inequality, we have
\begin{align*}
\nonumber
&\left \|\|\v\|^2(\w-\w^*)  - \tau  \tilde{\nabla}_{\w} L(\w) \right \|_\infty \\
&\le C_4 \|\v\|^2 \sqrt{\frac{n}{N}} + C_3\,\tau\,\vnorm{\v}_1^2 \sqrt{\frac{n + 2\ln m}{N}} + C_8 \tau K_\xi \|\v\|_1 \sqrt{\frac{n + 2\ln m}{N}} \\
&\le  (C' \|\v\|_1^2 + C''K_\xi\|\v\|_1)\sqrt{ { n \over N} } \quad \small({\text{since $m \lesssim O(2^n)$ by Assumption~\ref{assump:data} and $\|\v\| \le \|\v\|_1$}}) \normalsize
\end{align*}
\normalsize
for all $\w \in \mathbb{Q}_1$. Dividing both sides by $\tau$ yields the bound in Theorem~\ref{thm:RAIC_for_two_layer_quantized_neural_networks_appendix}. The success probability follows by taking the minimum of the positive constants $c_7, c_5, c_6$ in the lemmas such that
$1 - 3m \exp(-c_7 n) - \exp(-c_5 n) - \exp(-c_6 n) \ge 1 - 5m\exp(-c n)$, where $c$ is the minimum.
\end{proof}

\subsection{Finite-Sample Convergence Analysis}
Following Theorem \ref{thm:RAIC}, the first step of the STE-gradient method can be rewritten as:
\begin{align}
\nonumber
\x^t &= \x^{t-1} - \eta_t\tilde{\nabla}_{\w} L(\w^{t-1})  \\
\nonumber
&= \x^{t-1} - \eta_t \, {\|\v\|^2 \over \tau} (\w^{t-1}-\w^*) +  \eta_t \left( {\|\v\|^2 \over \tau} (\w^{t-1}-\w^*)  - \tilde{\nabla}_{\w} L(\w^{t-1})\right) \\
&= \x^{t-1} + \underbrace{\eta_t {\|\v\|^2 \over \tau} (\w^* - \w^{t-1})}_{\text{drift term}} +  \underbrace{\eta_t O\left( \sqrt{n \over N} \right)}_{\text{perturbation term}},  \label{eq:noisy_main_eq_2_body}
\end{align}
The key to understanding the convergence behavior, both in the ergodic and non-ergodic cases, lies in analyzing the component-wise dynamics governed by \eqref{eq:noisy_main_eq_2_body}, as outlined in the remainder of this section.

\subsubsection{Ergodic Convergence Analysis}
\label{sec:ergodic_constant_proof_main}


We first outline the proof of Theorem~\ref{thm:Informal_ergodic_informal_nonasymp} and then give the representative constant-step-size argument in detail. This inclusion is meant to make the occupation-time mechanism explicit; the extension to general step sizes and nonzero initialization is deferred to Appendix~\ref{sec:general_decaying_step_size_analysis_final_corrected_appendix}.

\paragraph{Drift dominates the perturbation when $N$ is large.} An important observation is that since $\w^t = \mathcal{Q}(\x^t)$, it holds that $\mathrm{sign}(w_j^t) = \mathrm{sign}(x_j^t)$, $\forall \; j\in[n]$. Whenever
$\mathrm{sign}(w_j^{t-1}) = \mathrm{sign}(x_j^{t-1})\neq \mathrm{sign}(w_j^*)$,
the drift term actively steers the dynamics of $\{x_j^t\}$ towards the desired sign, $\mathrm{sign}(w_j^*)$. Note that each component of the drift term has the scale of $O(1/\sqrt{n})$.
This drift term dominates the perturbation term when the perturbation-to-drift strength ratio, denoted by $\rho = O(n/\sqrt{N})$, is sufficiently small. Crucially, $\rho$ is governed by the concentration bound in Theorem~\ref{thm:RAIC} and scales inversely with $\sqrt{N}$, ensuring that the drift prevails over the perturbation when the sample size~$N$ is large enough.

\paragraph{Occupation time analysis.} Moreover, when the perturbation-to-drift strength ratio $\rho$ is sufficiently small, each coordinate $x_j^t$ exhibits a pattern of brief excursions into the incorrect sign region, followed by longer phases in the sign region of $w_j^*$. Crucially, both $w_j^t$ and $x_j^t$ spend proportionally more time in the correct sign region, which approximately a factor of $1/\rho$ longer than in the incorrect one. As a result, the average of $w_j^t$ over time $t$ converges to $w_j^*$, up to an error of $O(\rho/\sqrt{n}) = O(\sqrt{n/N})$. Therefore, $\|\overline{\w}^T - \w^*\|_\infty$ also has an error floor of $O(\sqrt{n/N})$. This is the mechanism formalized next in the constant-step-size setting and extended to the general step-size case in Appendix~\ref{sec:general_decaying_step_size_analysis_final_corrected_appendix}.

\medskip
\noindent We now make the above mechanism precise in the representative constant-step-size setting. The purpose of including this argument here is to show the core idea behind the ergodic analysis: the concentration theorem converts the STE-gradient dynamics into a deterministic drift system with a uniformly bounded perturbation; once the drift is stronger than the perturbation, incorrect-sign excursions become both short and infrequent. The full general-step-size result is obtained later in Appendix~\ref{sec:general_decaying_step_size_analysis_final_corrected_appendix} by extending this same occupation-time argument beyond constant step sizes and zero initialization.

\begin{theorem}[Non-Asymptotic Ergodic Convergence]
\label{thm:ergodic_informal_nonasymp}
Suppose that Assumptions~\ref{assump:data} and~\ref{assump:noise} hold. Assume the sample size $N$ is large enough such
\[
N \ge \max\left( C_1 n^2, \left(  \frac{5n (C' \|\v\|_1^2 + C''K_\xi \|\v\|_1)}{2\|\v\|^2} \right)^2  \right)
\]
to ensure that the perturbation-to-drift strength ratio $\rho := \frac{n (C' \|\v\|_1^2 + C''K_\xi \|\v\|_1)}{2\|\v\|^2 \sqrt{N}} < 1/5$ for some positive universal constants $C_1, C'$ and $C''$.
Then, with probability at least $1 - 5m\exp(-cn)$ for some positive universal constant $c > 0$, the following statements hold:

For a constant step size $\eta_0$ and $\x^0 = 0$, the ergodic average satisfies for a sufficiently large $T$:
\[
\left\| \overline{\w}^T - \w^* \right\|_\infty \le \frac{1}{\sqrt{n}} \cdot \frac{2 \rho}{1-\rho} + O\left(\frac{1}{T\sqrt{n}}\right).
\]
Moreover, if the number of iterations $T$ is large enough such  that $T > \frac{4(1-\rho)}{1-3\rho}$
the exact recovery is possible, by quantizing the ergodic average, i.e., \[
\mathcal{Q}(\overline{\w}^T) = \w^*
\]
\end{theorem}

\begin{rmk}
Note that the bound in Theorem~\ref{thm:ergodic_informal_nonasymp} for the ergodic average implies the bound
\[
\left\| \overline{\w}^T - \w^* \right\|_\infty \le \frac{1}{\sqrt{n}} \cdot \frac{2 \rho}{1-\rho} + O\left(\frac{1}{T\sqrt{n}}\right) = O\left(\sqrt{\frac{n}{N}}\right) + O\left(\frac{1}{\sqrt{n} T}\right),
\] in Theorem~\ref{thm:Informal_ergodic_informal_nonasymp} (the informal version) because $\rho = O\left({n \over \sqrt{N}} \right)$.
\end{rmk}

\begin{proof}[Proof of Theorem~\ref{thm:ergodic_informal_nonasymp}.]
The first step of the STE-gradient method can be rewritten as:
\begin{align}
\nonumber
\x^t &= \x^{t-1} - \eta_t\tilde{\nabla}_{\w} L(\w^{t-1})  \\
&= \x^{t-1} + \underbrace{\eta_t {\|\v\|^2 \over \tau} (\w^* - \w^{t-1})}_{\text{drift term}} +  \underbrace{\eta_t  \left( {\|\v\|^2 \over \tau} (\w^{t-1}-\w^*)  - \tilde{\nabla}_{\w} L(\w^{t-1})\right) }_{\text{perturbation term } \eta_t\boldsymbol{\epsilon}^{t}}. \label{eq:noisy_main_eq_2}
\end{align}

The unscaled perturbation term, $\boldsymbol{\epsilon}^{t}$, captures the approximation error given by Theorem \ref{thm:RAIC_for_two_layer_quantized_neural_networks_appendix}.
Suppose each component of $|\boldsymbol{\epsilon}^{t}|$ is bounded by
$$
\widetilde{\Delta}:=  {1 \over \tau} (C' \|\v\|_1^2 + C''K_\xi \|\v\|_1) \sqrt{ \frac{n}{N} }  = O((\|\v\|_1^2 + K_\xi \|\v\|_1) )\sqrt{ \frac{n}{N} }.
$$
This bound holds with probability at least $1 - 5m\exp(-cn)$ for some positive universal constant $c > 0$, as established by the $\ell_\infty$ concentration bound in Theorem~\ref{thm:RAIC_for_two_layer_quantized_neural_networks_appendix}.

The key to understanding the convergence behavior lies in analyzing the drift of the dynamics of a single coordinate $x_j^t$ governed by Equation \eqref{eq:noisy_main_eq_2}. 

The  drift term is $\eta_t {\|\v\|^2 \over \tau} \left(  w_j^* - \frac{1}{\sqrt{n}}\mathrm{sign}(x_j^{t-1}) \right)$ and its magnitude is $\eta_t {2 \|\v\|^2 \over \tau \sqrt{n}}$ when the $\frac{1}{\sqrt{n}}\mathrm{sign}(x_j^{t-1}) \neq w_j^*$ or the sign of $x_j^{t-1}$ differs from the sign of $w_j^*$. Thus, this justifies the perturbation-to-drift strength ratio $\rho$ definition in Theorem~\ref{thm:ergodic_informal_nonasymp} since
\[
\frac{\eta_t \cdot \widetilde{\Delta}}{\eta_t \cdot 2 \|\v\|^2 / (\tau \sqrt{n})} = \frac{(C'\|\v\|_1^2 +  C''K_\xi\|\v\|_1) n}{2\|\v\|^2 \sqrt{N}} = \rho
\]


To make the notation light and connect with the cycle analysis later, and since we consider the constant step size case in this section, we also define
\[
a := \frac{\eta_0 \|\v\|^2}{\tau \sqrt{n}} = {\eta_0\widetilde{\Delta} \over 2 \rho}.
\]

Before analyzing the zero-crossing behavior, let us recall the key parameters from our setup: $\rho$ captures the ratio between the upper bounds of the perturbation and drift magnitudes, and $2a$ is the size of the drift (including the step size $\eta_0$).

{\bf Step 1: Zero--Crossing Times.} This analysis uses the parameters $a$ and the perturbation-to-drift strength ratio $\rho$, where $\rho$ incorporates the  perturbation bound $\widetilde{\Delta}$ such that $\eta_0|\epsilon^t_{j}| \le \eta_0\widetilde{\Delta} = 2a\rho$.

First, assume that $w_j^* = 1/\sqrt{n}$ without loss of generality. The other case follows by symmetry.
Define the sequence of zero–crossing times for coordinate $j$:
\[
 t_1 < t_2 < t_3 < \cdots,
\]
where \(t_{2m-1}\) is the \(m\)th time \(x_j^t\) crosses from nonnegative ($\ge 0$) to negative ($<0$), and \(t_{2m}\) is the \(m\)th time it crosses from negative back to nonnegative.
When the process crosses from nonnegative to negative at time \(t_{2m-1}\), we have $x_j^{t_{2m-1}-1} \ge 0$. The update dynamics give:
\[
0 > x_j^{t_{2m-1}} = x_j^{t_{2m-1}-1} - \eta_0\,{\|\v\|^2 \over \tau}\Bigl(\frac{1}{\sqrt{n}}-(+\frac{1}{\sqrt{n}})\Bigr) + \eta_0\,\epsilon^{t_{2m-1}}_{j} = x_j^{t_{2m-1}-1} + \eta_0\,\epsilon^{t_{2m-1}}_{j}.
\]
Since $x_j^{t_{2m-1}-1} \ge 0$ and the perturbation term is bounded as $\epsilon^{t_{2m-1}}_{j} \ge -\widetilde{\Delta}$:
\[
x_j^{t_{2m-1}} \ge 0 - \eta_0\widetilde{\Delta} = -2a\rho.
\]
We show that the next crossing $t_{2m}$ occurs immediately, $t_{2m} = t_{2m-1}+1$. At time $t = t_{2m-1}+1$, the preceding state was $x_j^{t_{2m-1}}$, which is negative ($w_j^{t_{2m-1}}=-1/\sqrt{n}$). The dynamics give:
\begin{align*}
x_j^{t_{2m-1}+1} &= x_j^{t_{2m-1}} - \eta_0\,{\|\v\|^2 \over \tau}\Bigl(-\frac{1}{\sqrt{n}}-(+\frac{1}{\sqrt{n}})\Bigr) + \eta_0\,\epsilon^{t_{2m-1}+1}_{j} \\
&= x_j^{t_{2m-1}} + 2a + \eta_0\,\epsilon^{t_{2m-1}+1}_{j}.
\end{align*}

Using $x_j^{t_{2m-1}} \ge -2a\rho$ and the perturbation bound $\epsilon^{t_{2m-1}+1}_{j} \ge -\widetilde{\Delta}$:
\[
x_j^{t_{2m-1}+1} \ge (-2a\rho) + 2a - \eta_0\widetilde{\Delta} = -2a\rho + 2a - 2a\rho = 2a(1-2\rho).
\]
Since $\rho < 1/5 < 1/2$, we have $2a(1-2\rho) > 0$. Thus $x_j^{t_{2m-1}+1} > 0$, confirming $t_{2m} = t_{2m-1}+1$. Each negative event lasts exactly one time step.

Similarly, suppose that $w_j^* = -1/\sqrt{n}$. Then $t_{2m}$ is the $m$-th time $x_j^t$ crosses from negative ($<0$) to non-negative ($\ge 0$). This implies $x_j^{t_{2m}-1} < 0$ ($w_j^{t_{2m}-1} = -1/\sqrt{n}$). The update gives:
\[
0 \le x_j^{t_{2m}} = x_j^{t_{2m}-1} - \eta_0\,{\|\v\|^2 \over \tau}\Bigl(-\frac{1}{\sqrt{n}} - (-\frac{1}{\sqrt{n}})\Bigr) + \eta_0\,\epsilon^{t_{2m}}_{j} = x_j^{t_{2m}-1} + \eta_0\,\epsilon^{t_{2m}}_{j}.
\]
Since $x_j^{t_{2m}-1} < 0$ and the perturbation term is bounded $\epsilon^{t_{2m}}_{j} \le \widetilde{\Delta}$:
\[
x_j^{t_{2m}} \le 0 + \eta_0\widetilde{\Delta} = 2a\rho.
\]
Now consider the next step $t_{2m}+1$. We have $x_j^{t_{2m}} \ge 0$, so $w_j^{t_{2m}} = 1/\sqrt{n}$.
\[
 x_j^{t_{2m}+1} = x_j^{t_{2m}} - \eta_0\,{\|\v\|^2 \over \tau}\Bigl(+\frac{1}{\sqrt{n}} - (-\frac{1}{\sqrt{n}})\Bigr) + \eta_0\,\epsilon^{t_{2m}+1}_{j} = x_j^{t_{2m}} - 2a + \eta_0 \epsilon^{t_{2m}+1}_{j}.
\]
Using $x_j^{t_{2m}} \le 2a\rho$ and the perturbation bound $\epsilon^{t_{2m}+1}_{j} \le \widetilde{\Delta}$:
\[
x_j^{t_{2m}+1} \le (2a\rho) - 2a + \eta_0\widetilde{\Delta} = 2a\rho - 2a + 2a\rho = 2a(2\rho-1) = -2a(1-2\rho).
\]
Since $\rho < 1/5$, this is strictly negative. Thus $x_j^{t_{2m}+1} < 0$. This confirms the one-step reset property symmetrically.

In both cases, the coordinate resets to the ``correct" sign region within one time step of entering the ``incorrect" sign region, provided $\rho < 1/5$.

{\bf  Step 2: Occupation Time Analysis.} Without loss of generality, assume that $w_j^* = 1/\sqrt{n}$.
The dynamics of the $j$-th component $x_j^t$ exhibit the following cycle structure. Define a cycle as:
\begin{itemize}
    \item A \emph{negative phase}: one time step during which $x_j^t<0$ (with $\operatorname{sign}(x_j^t)=-1$).
    \item A \emph{positive phase}: a block of consecutive time steps during which $x_j^t\ge 0$ (with $\operatorname{sign}(x_j^t)=1$).
\end{itemize}
After a negative event at time $t$, we have $x_j^{t+1}\ge 2a(1-2\rho) > 0$ by the result of Step 1. To force the $j$-th component $x_j^k$ from this positive level back below zero, a cumulative downward displacement of at least $2a(1-2\rho)$ is required. Since each step in the positive phase may decrease $x_j^k$ by at most $2a\rho$ (in the worst case because the perturbation term $\eta_0|\epsilon^{t}_j| \le 2a\rho$), it takes at least
\[
L:=\left\lceil\frac{2a(1-2\rho)}{2a\rho}\right\rceil = \left\lceil\frac{(1-2\rho)}{\rho}\right\rceil
\] steps. Thus, a complete cycle consists of one negative step followed by at least $L$ positive steps, giving a minimum cycle length of $1+L$. We know that $L \ge (1-2\rho)/\rho$, which implies $1+L \ge 1 + (1-2\rho)/\rho = (1-\rho)/\rho$.

If $x_j^t\ge 0$ for every $t=1,\dots,T$, then $\frac{1}{T}\sum_{t=1}^T\operatorname{sign}(x_j^t)=1$, and the desired lower bound is immediate. Hence, assume that at least one negative event occurs. Let $T_0 = t_1$ be the first time in $[1,T]$ the iterate $x_j^t$ becomes negative. By definition, this means $x_j^t \ge 0$ for $t=1, \dots, T_0-1$, and $x_j^{T_0} < 0$.
Let $\gamma$ be the number of negative events (steps where $x_j^t < 0$) in the interval $[T_0+1, T]$. Let $T' = T - T_0$. Each negative step in $[T_0+1, T]$ initiates a cycle of minimum length $1+L$. The number of full cycles starting in $[T_0+1, T]$ is at most $\lfloor T' / (1+L) \rfloor$. The remaining partial cycle can contain at most one negative step. Therefore, $\gamma \le \lfloor T' / (1+L) \rfloor + 1$.
Using the inequality $1+L \ge (1-\rho)/\rho$, we have $1/(1+L) \le \rho/(1-\rho)$. Substituting this into the bound for $\gamma$:
\[
\gamma \le \frac{T'}{1+L} + 1 \le T' \frac{\rho}{1-\rho} + 1
\]


{\bf Step 3: Bounding the Average Sign.} Now, we find a lower bound for the sum of signs over the entire interval $[1, T]$:
\[
\sum_{t=1}^T \operatorname{sign}(x_j^t) = \sum_{t=1}^{T_0} \operatorname{sign}(x_j^t) + \sum_{t=T_0+1}^{T} \operatorname{sign}(x_j^t)
\]
For the term part, using the definition of $T_0$, we have
\[
\sum_{t=1}^{T_0} \operatorname{sign}(x_j^t) = \sum_{t=1}^{T_0-1} (+1) + \operatorname{sign}(x_j^{T_0}) = (T_0 - 1)  + (-1) = T_0 - 2
\]
For the second part, let $p$ be the number of non-negative steps ($x_j^t \ge 0$) in the interval $[T_0+1, T]$. We have $\gamma+p = T'$. The sum over this interval is given by:
\[
\sum_{t=T_0+1}^T \operatorname{sign}(x_j^t) = \gamma \times (-1) + p \times (+1) = -\gamma + (T' - \gamma) = T' - 2\gamma
\]
Using the upper bound $\gamma \le T' \frac{\rho}{1-\rho} + 1$ yields
\[
\sum_{t=T_0+1}^T \operatorname{sign}(x_j^t) = T' - 2\gamma \ge T' - 2 \left( T' \frac{\rho}{1-\rho} + 1 \right) = T' \left( 1 - \frac{2\rho}{1-\rho} \right) - 2
\]
Combining the two parts for the total sum:
\[
\sum_{t=1}^T \operatorname{sign}(x_j^t) \ge (T_0 - 2) + \left[ T' \left( 1 - \frac{2\rho}{1-\rho} \right) - 2 \right]
\]
We substitute $T' = T - T_0$:
\[
\sum_{t=1}^T \operatorname{sign}(x_j^t) \ge T_0 - 2 + (T - T_0) \left( 1 - \frac{2\rho}{1-\rho} \right) - 2
\]
Thus, we have
\begin{align}
\label{eq:sum_bound}
\sum_{t=1}^T \operatorname{sign}(x_j^t) \ge T - (T - T_0) \frac{2\rho}{1-\rho} - 4
\end{align}

Finally, we derive the lower bound for the average sign over the interval $[1, T]$ by dividing \eqref{eq:sum_bound} by $T$:
\[
\frac{1}{T} \sum_{t=1}^T \operatorname{sign}(x_j^t) \ge \frac{1}{T} \left[ T - (T - T_0) \frac{2\rho}{1-\rho} - 4 \right] = 1 - \frac{T - T_0}{T} \frac{2\rho}{1-\rho} - \frac{4}{T}
\]
Since $1 \le T_0 \le T$, we have $T - T_0 < T$, which implies $\frac{T-T_0}{T} < 1$ for all $T_0 \ge 1$. As $\frac{2\rho}{1-\rho} > 0$ for $0 < \rho < 1/2$, we can bound the term: $-\frac{T-T_0}{T} \frac{2\rho}{1-\rho} \ge -(1) \frac{2\rho}{1-\rho}$.
Therefore,
\[
\frac{1}{T} \sum_{t=1}^T \operatorname{sign}(x_j^t) \ge 1 - \frac{2\rho}{1-\rho} - \frac{4}{T}
\]
By the same argument and symmetry for the case $w_j^* = -1/\sqrt{n}$, with the possible initial mismatch due to the convention $\operatorname{sign}(0)=1$ absorbed into the $O(1/T)$ term, we have
\[
\frac{1}{T} \sum_{t=1}^T \operatorname{sign}(x_j^t) \le -1 + \frac{2\rho}{1-\rho} + \frac{4}{T}
\]

Since the component-wise error bound
\[
|w_j^* - \overline{w}_j^T| \le \frac{1}{\sqrt{n}} \left( \frac{2\rho}{1-\rho} + \frac{4}{T} \right)
\]
holds uniformly for all components $j=1, \dots, n$, as established before, we have:
\begin{align}
\label{eq:final_linf_bound}
\left\|\w^* -  \overline{\w}^T\right\|_\infty \le \frac{1}{\sqrt{n}} \left( \frac{2\rho}{1-\rho} + \frac{4}{T} \right)
\end{align}
This bound shows that the ergodic average of the iterates converges to a ball around the target $\w^*$ in the $\ell_\infty$ norm, with the radius of the ball shrinking as the perturbation-to-drift strength ratio  $\rho$ decreases, plus a term that vanishes as the number of iterations $T$ increases.

\paragraph{Exact Recovery Condition.}
Exact coordinate-wise sign recovery requires that the sign of the average iterate matches the target sign for each component:
\[
\operatorname{sign}\!\Bigl(\frac{1}{T}\sum_{t=1}^T w_j^t\Bigr)
\;=\;
\operatorname{sign}(w_j^*) \quad \text{for all } j=1, \dots, n,
\] which is equivalent to
\[
\mathcal{Q}\left(\frac{1}{T}\sum_{t=1}^T \w^t\right)
\;=\;
\w^*.
\]
This is guaranteed if $\|\w^* - \tfrac{1}{T}\sum_{t=1}^T \w^t\|_\infty < \tfrac{1}{\sqrt{n}}$. Using the bound derived in \eqref{eq:final_linf_bound}, this condition is satisfied if $
\frac{1}{\sqrt{n}} \left( \frac{2\rho}{1-\rho} + \frac{4}{T} \right) < \frac{1}{\sqrt{n}}$,
which is equivalent to the condition $\frac{2\rho}{1-\rho} + \frac{4}{T} < 1$; this condition simplifies to:
\[
\frac{4}{T} < \frac{1-3\rho}{1-\rho}
\]
Hence, from the assumption $\rho < 1/5 < 1/3$, the exact recovery is achieved when
\[
T > \frac{4(1-\rho)}{1-3\rho}.
\]
\end{proof}

\subsubsection{Last-Iterate Convergence Analysis}
When $N$ is sufficiently large, we can show that the last iterate $w_j^t$ exactly reaches $w_j^*$ and then departs from it (assuming the noise $\xi \not\equiv 0$). The proof of Theorem \ref{thm:global_coordinate_recovery_noisy} is provided in Appendix \ref{sec:Non_Ergodic_Behavior_appendix} (constant step size) and Appendix \ref{sec:general_decaying_step_size_analysis_final_corrected_appendix} (general step size).

\paragraph{Visit to $\w^*$} There exists a time $t$ at which the signs of all components of $\x^t$ simultaneously match those of $\w^*$. Intuitively, this occurs when the perturbation-to-drift strength ratio $\rho$ is sufficiently small, allowing the drift to reliably steer each coordinate toward its correct sign. Note that achieving a simultaneous sign match across all components is significantly more difficult than achieving component-wise agreement. This explains why last-iterate convergence requires more samples than ergodic convergence.

\paragraph{Escape from $\w^*$} When label noise $\xi$ is present, the STE-gradient $\tilde{\nabla}_{\w} L(\w^*)$ does not vanish. Leveraging this fact, we show that the iterates $\w^t$ repeatedly escape from $\w^*$ with probability roughly at least $1/2$, over the randomness of the sample draw. We remark that our predicted escape probability of $1/2$ is rather conservative, and we believe it can be much higher in practice.

\section{Numerical Experiments}\label{sec:experiments}

\subsection{Gaussian Data}

\paragraph{Sample Complexity.} We empirically check the sample complexity required to obtain the optimal binary weights $\w^*$, in both non-ergodic and ergodic scenarios. Fixing $m=128$, for each dimension $n$, we run Algorithm 1 across various sampling ratios $N/n$. At each sampling level, we generate 100 random instances of $\left(\{\Z_i\}_{i=1}^N, \w^*, \v \right)$, and the noises $\xi^{(i)}$'s are set to 0, i.e., no label noise. We record the recovery rate of successfully finding $\w^*$ via STE-based optimization. We run $T = 500$ iterations and define success as achieving $\mathcal{Q}(\overline{\w}^T) = \w^*$ in the ergodic convergence case, and $\w^T = \w^*$ in the non-ergodic (last-iterate) case.

The left and middle panels of Figure \ref{fig} show that the transition boundaries for the success rates in both cases appear to be linear, suggesting that the required sample complexity $N$ scales quadratically with $n$ in both cases. Therefore, our theoretical sample complexity bound in the ergodic convergence case given by Theorem \ref{thm:Informal_ergodic_informal_nonasymp} appears to be sharp, whereas in the non-ergodic case, Theorem \ref{thm:global_coordinate_recovery_noisy} falls short of achieving the empirically observed (near-)quadratic scaling, indicating that our theoretical estimate seems to be conservative in this case. However, the plots also show that the sample complexity required for last-iterate convergence is higher than that for ergodic convergence, aligning with our theoretical findings.

\paragraph{Recurrence Behavior.} We also empirically observe the phenomenon of the recurrence behavior of the last iterate $\w^T$ around $\w^*$ under noisy labels, as predicted by Theorem \ref{thm:global_coordinate_recovery_noisy}. In this case, we set $m=128, n = 25, N = 140$ and add Gaussian noises $\{\xi^{(i)}\}_{i=1}^N$ with unit variance to the labels generated according to (\ref{eq:labels}). The right penal of Figure \ref{fig} shows that $\{\w^t\}_{t\geq1}$ indeed repeatedly visits and escapes from the minimum $\w^*$, which is also consistent with our theory, i.e., Theorem \ref{thm:global_coordinate_recovery_noisy}.


\begin{figure}[h!]
\setlength{\tabcolsep}{0.5pt}
    \centering
    \begin{tabular}{ccc}
    \includegraphics[width=0.33\linewidth]{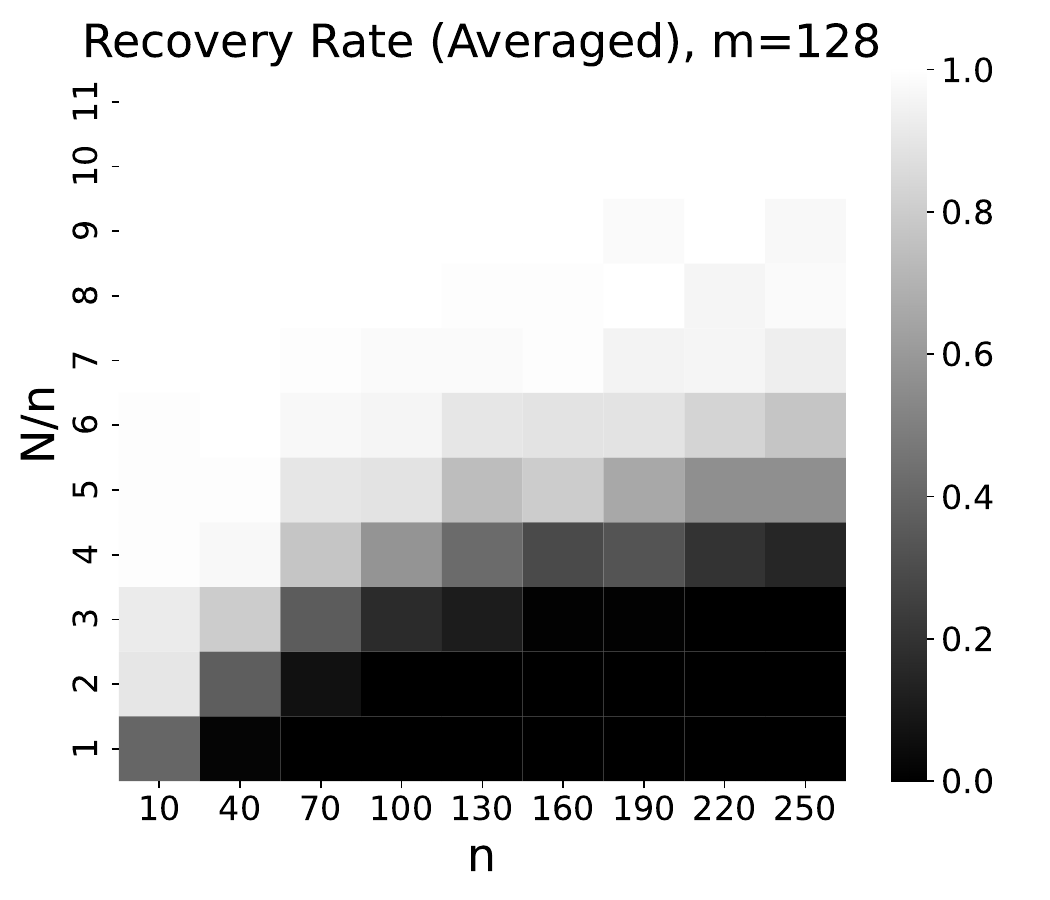} &
    \includegraphics[width=0.33\linewidth]{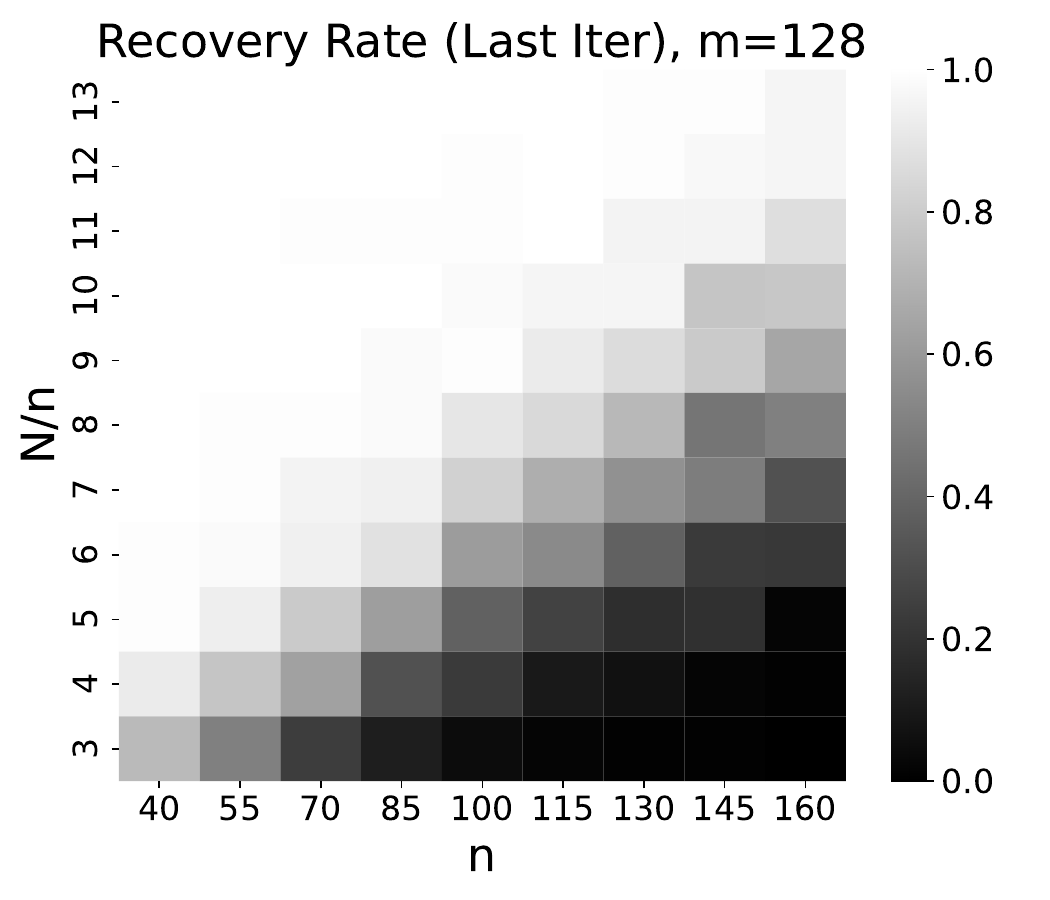} &
    \includegraphics[width=0.33\linewidth]{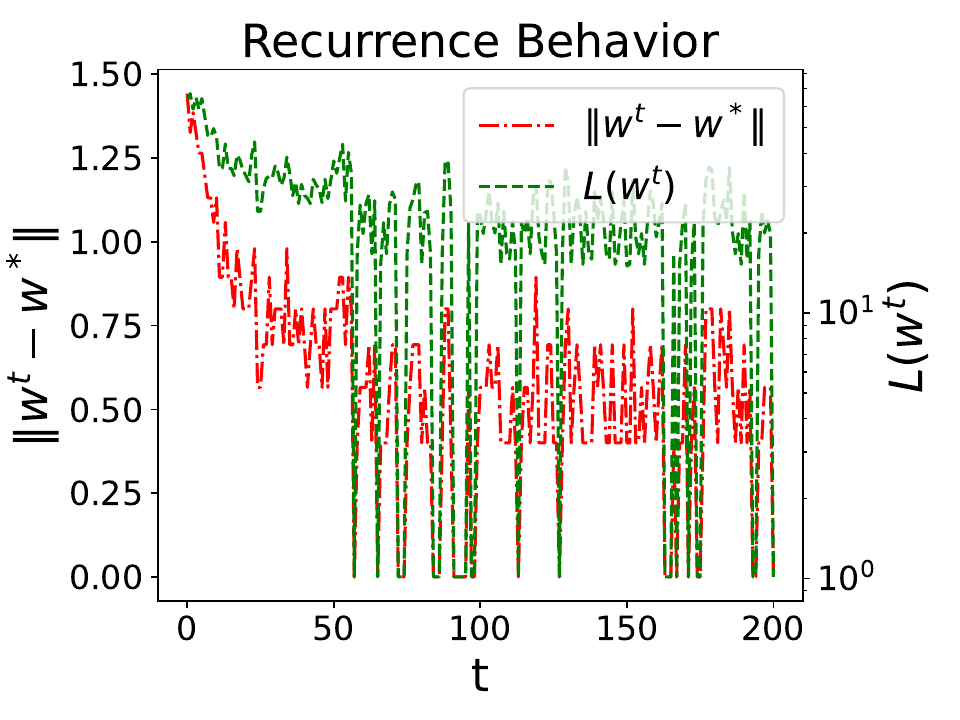}
    \end{tabular}
    \caption{{\bf Left:} Recovery rate for the ergodic averaged iterate $\mathcal{Q}(\overline{\w}^T)$ in the noiseless case. {\bf Middle:} Recovery rate for the last-iterate $\w^T$ in the noiseless case. {\bf Right:} Curves for $\|\w^t-\w^*\|$ vs. $t$ and $L(\w^t)$ vs. $t$ in the noisy case with $m=128, n = 25, N = 140$.} 
    \label{fig}
\end{figure}

\subsection{Non-Gaussian Data and Normalization}
Although the Gaussian data assumption may seem restrictive and unrealistic, we empirically show that the STE can fail to perform reliably for general non-Gaussian distributed data, raising the question of why this heuristic remains prevalent in practical quantization. To address this, we demonstrate that applying a normalization procedure significantly enhances the STE's performance, restoring its effectiveness across broader data distributions. 

\paragraph{Sub-Gaussian Data.} We further examine the behavior of Algorithm~\ref{alg:ste} on sub-Gaussian data. Our experiments indicate that its convergence guarantees deteriorate significantly for general sub-Gaussian distributions. In particular, under the uniform distribution $\mathcal{U}[0,1]$ and the Bernoulli distribution $\mathcal{B}(0.5)$ with $\mathrm{Pr}(X=0)=\mathrm{Pr}(X=1)=0.5$, both of which possess non-zero mean, the STE fails to consistently converge to an optimal solution (with a zero loss) as the sample size increases. As shown in Figure~\ref{fig:nongaussian}, the success rates (over 100 randomized runs) remain low and do not exhibit improvement with larger sample sizes $N$. When the data are normalized by subtracting the mean and dividing by the standard deviation, convergence improves notably; this corresponds to $\sqrt{12}\,(X-0.5)$ for $X \sim \mathcal{U}[0,1]$ and $2X-1$ for $X \sim \mathcal{B}(0.5)$.

\paragraph{MNIST Classification.} These findings emphasize the crucial role of normalization in ensuring the effectiveness of STE-based training. To further substantiate this observation beyond the weight recovery setting, we conduct experiments on MNIST classification using a binary CNN comprising two convolutional layers (Conv) with $5\times5$ kernels and a single fully connected (FC) layer:
$$
\mathrm{Conv} (5\times 5\times6)  \to \mathrm{MaxPool} (2\times 2) \to
 \mathrm{Conv} (5\times 5\times 16) \to  \mathrm{MaxPool} (2\times 2) \to
\mathrm{FC}.
 $$
In this configuration, both the weights and activations of the convolutional layers are binarized with an adaptive multiplicative full-precision scalar, following common practice \cite{rastegari2016xnor, yin2018blended}, while the fully connected layer is retained in full precision. The cross-entropy loss is used, and the ReLU STE is adopted for activation quantization during training. We observe that, in the absence of normalization—applied neither to the input data (entering the first convolutional layer) nor to the intermediate features (entering the second convolutional layer), the training accuracy plateaus around 30\%. In contrast, applying normalization for both the input data and intermediate features (via a non-affine batch normalization layer\footnote{The standard batch normalization with trainable affine parameters will remain effective in this case.} \cite{bnorm_15,halfwave_17}) substantially improves performance, achieving approximately 96\% accuracy.

\medskip

\noindent Taken together, our results validate the practical significance of the Gaussian data assumption central to our theoretical analysis and underscore the essential role of normalization in effective quantization.

\begin{figure}
    \centering
    \begin{tabular}{cc}
        \includegraphics[width=0.46\linewidth]{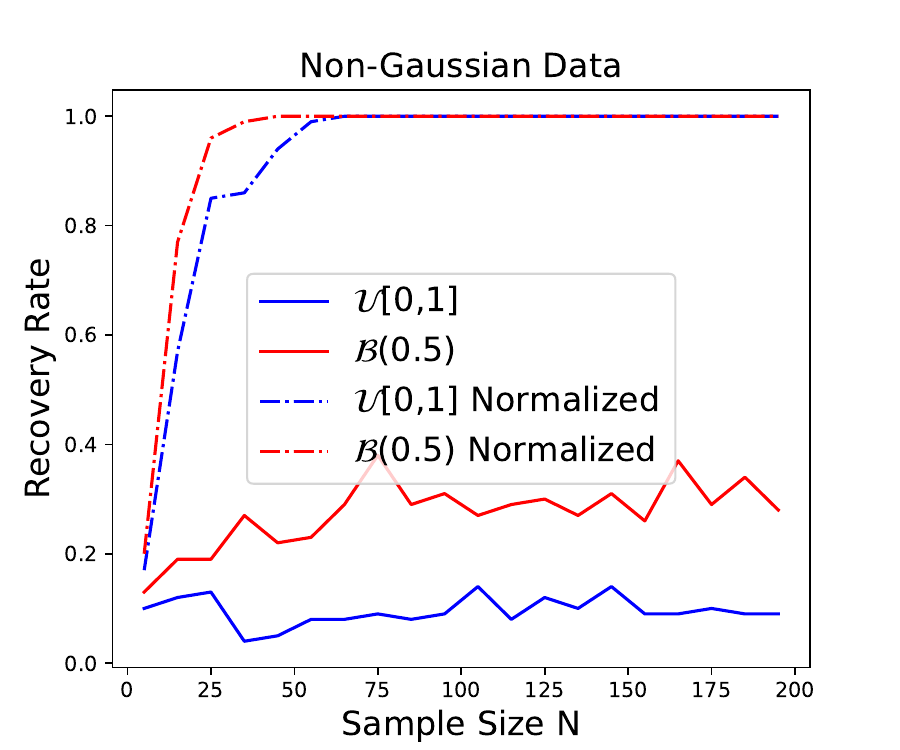} &
         \includegraphics[width=0.46\linewidth]{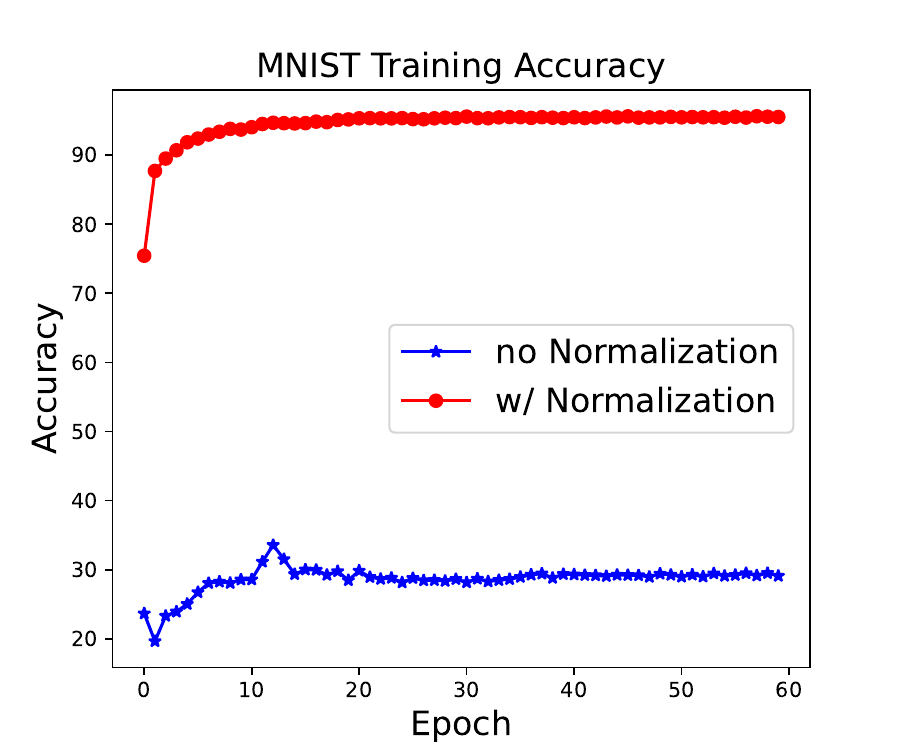}
    \end{tabular}
    \caption{{\bf Left:} Plot of recovery rate vs. sample size $N = 5, 15, \dots, 195,$ for uniformly and Bernoulli distributed data, $\mathcal{U}[0,1]$ and $\mathcal{B}(0.5)$, respectively, using Algorithm~\ref{alg:ste}. The data dimensions are fixed as $m = 32$, $n = 9$.  {\bf Right:} Training accuracy curves of a binarized CNN, comparing models trained with and without normalization applied to the input data and intermediate features.}
    \label{fig:nongaussian}
\end{figure}

\section{Concluding Remarks}
We presented the first sample complexity analysis for STE-based training for quantization-aware learning, emphasizing the pivotal role of sample size in ensuring its effectiveness. Our analysis focuses on a two-layer neural network with both binary weights and binary activations. The associated STE-gradient method employs dual STEs to handle the discreteness in both the objective and the constraints. We proved that $O(n^2)$ samples are sufficient to guarantee ergodic convergence to the optimal binary weights $\w^*$, while $O(n^4)$ samples suffice for a non-ergodic guarantee, where the iterates exactly reach $\w^*$ infinitely often. Notably, the use of STE generally leads the iterates to repeatedly diverge from the minimizer in weight quantization; however, we argue that this behavior is actually beneficial, as it prevents the premature stagnation of iterates. In addition, we empirically demonstrate that convergence guarantees break down under general non-Gaussian data distributions but can be restored through normalization.

\section*{Acknowledgments}
We thank Prof. Simon Foucart (Texas A\&M University) for inspiring discussions. This work was partially supported by NSF grants IIS-2110546, DMS-2151235, DMS-2208126, DMS-2219904, DMS-2309520, and a Qualcomm Gift Award.

\bibliographystyle{siamplain}
\bibliography{reference}

\newpage
\appendix

\paragraph{\Large{Supplementary Materials}} \mbox{}


\section{Concentration Bound for STE-Gradient}\label{sec:bound_STE}
\label{sec:Finite_sample_approximation_of_coarse_gradient_by_RAIC_appendix}
The formal STE-gradient concentration theorem is stated and proved in Section~\ref{sec:concentration}. We provide here the auxiliary estimates used in that proof.

\subsection{Concentration Bound for the Squared Term}\label{sec:squared_term}
Define
\begin{align*}
\widetilde{g}^{(j)}(\w)
&= \frac{1}{N}
\sum_{i=1}^N
   \Bigl(\1_{\{\z_j^{(i)\top} \w>0\}} - \1_{\{\z_j^{(i)\top} \w^*>0\}}\Bigr)  \Bigl(\1_{\bigl\{\z_j^{(i)\top} \w>0\bigr\}}\,\z_j^{(i)}\Bigr)\;,
\end{align*}
then the squared term  of the noiseless part of STE-gradient is given by $\sum_{j=1}^m
   v_j^2\,\widetilde{g}^{(j)}(\w)$.

To make the notation light, we define
\[
D(\w;\z^{(i)})
  \;:=\; \1_{\{\z^{(i)\top} \w>0\}}
          - \1_{\{\z^{(i)\top} \w^*>0\}}
\]  and
\[
\widetilde{g}(\w) =
\frac{1}{N}\sum_{i=1}^N D(\w;\z^{(i)}) \1_{\bigl\{\z^{(i)\top} \w>0\bigr\}}\,\z^{(i)}.
\]

The single-row RAIC estimate and the squared-term corollary are stated and proved in Section~\ref{sec:concentration}. We now give the three auxiliary concentration estimates used in the proof of Theorem~\ref{thm:RAIC_single_term}.

\subsubsection{Concentration of the first term}

\label{section:Concentration_of_main_term_1}
\begin{lemma}
\label{lem:Concentration_of_main_term_1}
Suppose that Assumptions~\ref{assump:data}, \ref{assump:noise}, and~\ref{assump:sample_size} hold. Let $\w^* \in \mathbb{Q}_1$  and $\tau = 2\sqrt{2\pi}$. Then, there exist positive universal constants $c_2$ and $C_3$ such that with probability at least $1 - \exp(-c_2 n)$, we have
    \[
     \left | \|\w - \w^*\| - {\tau \over N}  \sum_{i=1}^N \Bigl(D(\w;\z^{(i)})\Bigr) \1_{\bigl\{\z^{(i)\top} \w>0\bigr\}}  \inner{\z^{(i)},  {\w - \w^* \over \|\w - \w^*\|} }  \right |  \le C_3 \tau \sqrt{n \over N},
    \] for all $\w\in \mathbb{Q}_1 $ with $\w \neq \w^*$. Here $D(\w;\z^{(i)}) = \1_{\{\z^{(i)\top} \w>0\}} - \1_{\{\z^{(i)\top} \w^*>0\}}$.
\end{lemma}
\begin{proof}[Proof of Lemma~\ref{lem:Concentration_of_main_term_1}.]
Let $\u_1 \;=\;{\w - \w^* \over \|\w - \w^*\|} \,\in\mathbb S^{n-1}$ and let
\[
X_i(\w)
\;=\;
D(\w;\z^{(i)}) \,
\1_{\bigl\{\z^{(i)\top} \w>0\bigr\}} \,\z^{(i)\top} \u_1,
\]
then, we have
\[
{1 \over N}  \sum_{i=1}^N D(\w;\z^{(i)}) \1_{\bigl\{\z^{(i)\top} \w>0\bigr\}}  \inner{\z^{(i)},  {\w - \w^* \over \|\w - \w^*\|} }
\;=\;
\frac{1}{N}\sum_{i=1}^N X_i(\w).
\]
From Lemma 1 in \cite{long2023recurrence} (see, also \cite{yin2018blended}), we have
\begin{align*}
\mathbb{E}_{\Z} \left[D(\w;\z^{(i)}) \,
\1_{\bigl\{\z^{(i)\top} \w>0\bigr\}} \,\z^{(i)} \right] = {1 \over 2\sqrt{2\pi}} (\w-\w^*),
\end{align*}
so taking the inner product with \(\u_1\) and the linearity of expectation gives
\[
\mathbb{E}_{\Z} [X_i(\w)] = \u_1^\top \mathbb{E}_{\Z} \left[D(\w;\z^{(i)}) \,
\1_{\bigl\{\z^{(i)\top} \w>0\bigr\}} \,\z^{(i)}\right]
\;=\;
\frac{\u_1^\top(\w-\w^*)}{2\sqrt{2\pi}}
\;=\;
\frac{\|\w-\w^*\|}{2\sqrt{2\pi}}.
\]
Since $|D(\w;\z^{(i)})| \le 1$, we have
\[
|X_i(\w)|
\;\le\;
\1_{\{\z^{(i)\top} \w>0\}}\;|\z^{(i)\top} \u_1|
\;\le\;
|\z^{(i)\top} \u_1|.
\]
Since $\z^{(i)}$ is the standard Gaussian, $\|\z^{(i)\top} \u_1 \|_{\psi_2} \le C$ for some universal constant $C > 0$ uniformly for all $\w \in \mathbb{Q}_1\setminus\{\w^*\}$. So $X_i(\w)$ is also sub-Gaussian with sub-Gaussian norm uniformly bounded by $C$ for all $\w \in \mathbb{Q}_1\setminus\{\w^*\}$.
Since $\|X_i(\w)\|_{\psi_2} \le C$, $\|X_i(\w) -   \mathbb{E}_{\Z} [X_i(\w)]\|_{\psi_2} \le C'$ for some other universal constant $C' > 0$ for all $\w \in \mathbb{Q}_1\setminus\{\w^*\}$ by  \cite{vershynin2018high}.

By the generalized Hoeffding's bound \cite{vershynin2018high}, for any \(t>0\),
\[
\Pr\left(\left|\frac1N\sum_{i=1}^N (X_i(\w) - \mathbb{E}_{\Z}[X_i(\w)]) \right|>t\right)
\;\le\;
2\exp\bigl(-c\,N\,t^2 / (C')^2 \bigr),
\] for some universal constant $c > 0$.
Now choose a numeric constant \(C''>0\) such that $C'' > C'\sqrt{(2\ln 2)/c}$, and thus
\[
c \, N (C''\sqrt{n/N})^2 / (C')^2 = c \, n \, C''^2 / (C')^2 > (2\ln 2) n.
\]
If we set
\[
t = C''\sqrt{\frac{n}{N}} \quad \Longleftrightarrow \quad c N t^2 / (C')^2 > (2\ln 2) n,
\]
then with probability at least \(1 - 2e^{-c'''n}\) where \(c''':=c(C'')^2/(C')^2 > 2\ln2\).
Substituting $t$ and $\mathbb{E}_{\Z}[X_i(\w)] = \frac{\|\w-\w^*\|}{2\sqrt{2\pi}}$ in the Hoeffding's bound, we have
\[
\left|\frac1N\sum_{i=1}^N X_i(\w)
-\frac{\|\w-\w^*\|}{2\sqrt{2\pi}}\right|
\;\le\;
C''\sqrt{\frac{n}{N}},
\]
or
\[
\left|\|\w-\w^*\|
-\frac{\tau}{N}\sum_{i=1}^N X_i(\w)\right|
\;\le\;
C''\,\tau\sqrt{\frac{n}{N}}.
\]
Let $C_3 = C''$. Since $\mathbb{Q}_1$ is a finite set with cardinality \(|\mathbb{Q}_1|=2^n\), a union bound over all \( \w \in \mathbb{Q}_1, \w \neq \w^*\) (at most $2^n-1$ possibilities) gives total failure probability
\[
< 2^n\;\times\;2e^{-c'''n}
\;=\;
2\exp\bigl(n\ln2 - c'''n\bigr)
\;=\;
2\exp\bigl(-(c'''-\ln2)\,n\bigr).
\]
Since \(c'''>2\ln2\), we have \(c'''-\ln2> \ln2\), and the failure probability decays as $\exp(-\Omega(n))$. Therefore, with probability at least $1 - \exp(-\Omega(n))$, the inequality
\[
\left|\|\w-\w^*\|
-\frac{\tau}{N}\sum_{i=1}^N X_i(\w)\right|
\;\le\;
C_3\,\tau\sqrt{\frac{n}{N}}
\]
holds simultaneously for all $\w\in\mathbb{Q}_1\setminus\{\w^*\}$, which completes the proof.
\end{proof}

\subsubsection{Concentration of the second term}
\begin{lemma}[Concentration of the second term]
\label{lem:Concentration_of_main_term_2}
Suppose that Assumptions~\ref{assump:data}, \ref{assump:noise}, and~\ref{assump:sample_size} hold. Let $\w^* \in \mathbb{Q}_1$  and $\tau = 2\sqrt{2\pi}$. Then there exist positive universal constants $c_2'$ and $C_3'$, with probability at least $1 - \exp(-c_2' n)$, such that
    \[
     \left | {\tau \over N}  \sum_{i=1}^N D(\w;\z^{(i)}) \1_{\bigl\{\z^{(i)\top} \w>0\bigr\}}  \inner{\z^{(i)},  {\w + \w^* \over \|\w + \w^*\| } }  \right |  \le C_3' \tau \sqrt{n \over N},
    \]
    for all $\w \in \mathbb{Q}_1\setminus\{-\w^*\}$. Here $D(\w;\z^{(i)}) = \1_{\{\z^{(i)\top} \w>0\}} - \1_{\{\z^{(i)\top} \w^*>0\}}$.
\end{lemma}

\begin{proof}[Proof of Lemma~\ref{lem:Concentration_of_main_term_2}.]
Let $\u_2 \;=\;{\w + \w^* \over \|\w + \w^*\|}$. Since $\w, \w^* \in \mathbb{Q}_1$ and $\w \ne - \w^*$, the denominator is non-zero and $\u_2 \in \mathbb{S}^{n-1}$. Define the random variable:
\[
Y_i(\w) \;=\; D(\w;\z^{(i)}) \, \1_{\bigl\{\z^{(i)\top} \w>0\bigr\}} \,\z^{(i)\top} \u_2,
\]
The expression inside the absolute value in the lemma statement is $\frac{\tau}{N}\sum_{i=1}^N Y_i(\w)$.

Since $\mathbb{E}_{\Z} \left[D(\w;\z^{(i)}) \1_{\bigl\{\z^{(i)\top} \w>0\bigr\}} \,\z^{(i)}\right] = \frac{1}{2\sqrt{2\pi}} (\w-\w^*)$, by linearity of expectation,
\small
\[
\mathbb{E}_{\Z} [Y_i(\w)] \;=\; \inner{\mathbb{E}_{\Z} [D(\w;\z^{(i)}) \1_{\bigl\{\z^{(i)\top} \w>0\bigr\}} \,\z^{(i)}], \u_2} \;=\; \inner{{1 \over 2\sqrt{2\pi}} (\w-\w^*), {\w + \w^* \over \|\w + \w^*\|}}.
\]
\normalsize
Because $\|\w\| = \|\w^*\|$ for $\w, \w^* \in \mathbb{Q}_1$, the vectors $\w-\w^*$ and $\w+\w^*$ are orthogonal: $\inner{\w-\w^*, \w+\w^*} = \|\w\|^2 - \|\w^*\|^2 = 0$. Therefore,
\[
\mathbb{E}_{\Z} [Y_i(\w)] = 0.
\]
Next, we establish that $Y_i$ is sub-Gaussian. Since $|D(\w;\z^{(i)})| \le 1$, we have $|Y_i(\w)| \;\le\; |\z^{(i)\top} \u_2|$.
As $\z^{(i)} \sim \mathcal{N}(0, I_n)$ and $\u_2$ is a unit vector, $\z^{(i)\top} \u_2 \sim \mathcal{N}(0,1)$, which is sub-Gaussian. Since $|Y_i|$ is bounded by the magnitude of a standard sub-Gaussian random variable, $Y_i$ is itself sub-Gaussian with a sub-Gaussian norm $\|Y_i(\w)\|_{\psi_2} \le C'$ uniformly bounded by an universal constant $C'$ for all $\w \in \mathbb{Q}\setminus\{-\w^*\}$.

Since $Y_i(\w)$ are independent, mean-zero, sub-Gaussian random variables with uniformly bounded sub-Gaussian norm $C'$, we apply Hoeffding's bound \cite{vershynin2018high}, for any $t>0$:
\[
\Pr\left(\left|\frac{1}{N}\sum_{i=1}^N Y_i(\w) \right|>t\right)
\;\le\;
2\exp\bigl(-c\,N\,t^2 / (C')^2 \bigr),
\]
for some universal constant $c> 0$. Choose a constant $C''>0$ large enough such that $c''' := c(C'')^2/(C')^2 > 2\ln 2$. Let
\[
t \;=\; C''\sqrt{\frac{n}{N}}.
\]
Then, for a fixed $\w$ with $\w \ne - \w^*$, with probability at least $1 - 2e^{-c'''n}$:
\[
\Bigl|\frac{1}{N}\sum_{i=1}^N Y_i(\w)\Bigr| \;\le\; C''\sqrt{\frac{n}{N}}.
\]
Multiplying by $\tau = 2\sqrt{2\pi}$ yields:
\[
\Bigl|\frac{\tau}{N}\sum_{i=1}^N Y_i(\w)\Bigr| \;\le\; C''\,\tau\sqrt{\frac{n}{N}}.
\]
Let $C_3' = C''$. To ensure the above holds uniformly for all  $\w \in \mathbb{Q}_1\setminus\{-\w^*\}$, we use a union bound. The total number of such vectors is less than $2^{n}$. The total failure probability is bounded by
\[
< 2^{n}\;\times\;2e^{-c'''n}
\;=\;
2\exp\bigl(n\ln2 - c'''n\bigr)
\;=\;
2\exp\bigl(-(c'''-\ln2)\,n\bigr).
\]
Since $c'''>2\ln2$, let $c_2' = c''' - \ln 2 > 0$. The failure probability decays as $\exp(-\Omega(n))$. The condition $N \ge C_1' n^2$ ensures $t$ is appropriately scaled for the concentration regime. Therefore, with probability at least $1 - \exp(-c_2' n)$ (adjusting constant for the factor of 2), the inequality
\[
\Bigl|\frac{\tau}{N}\sum_{i=1}^N Y_i(\w)\Bigr| \;\le\; C_3'\,\tau\sqrt{\frac{n}{N}}
\]
holds simultaneously for all vectors $\w \in \mathbb{Q}_1$, completing the proof.
\end{proof}

\subsubsection{Concentration of the third term}
Recall that from the decomposition of $\z^{(i)}$, we have $\b_i(\w) = \z^{(i)} - \inner{\z^{(i)}, \u_1}\u_1 - \inner{\z^{(i)}, \u_2}\u_2$ is the component of $\z^{(i)}$ orthogonal to $\mathrm{span}\{\w-\w^*, \w+\w^*\}$, where $\u_1 = (\w-\w^*)/\|\w-\w^*\|$ and $\u_2 = (\w+\w^*)/\|\w+\w^*\|$. Recall also that $D(\w;\z^{(i)}) = \1_{\{\z^{(i)\top} \w>0\}} - \1_{\{\z^{(i)\top} \w^*>0\}}$.
Define the vector $R(\w) \in \mathbb{R}^n$ where
\[ R_j(\w) = \frac{\tau}{N} \sum_{i=1}^N D(\w;\z^{(i)}) \1_{\bigl\{\z^{(i)\top} \w>0\bigr\}} \inner{\b_i(\w), \e_j}. \]
\begin{lemma}
\label{lem:Concentration_of_main_term_3}
Suppose that Assumptions~\ref{assump:data}, \ref{assump:noise}, and~\ref{assump:sample_size} hold. Let $\w^* \in \mathbb{Q}_1$  and $\tau = 2\sqrt{2\pi}$. Then, there exist positive universal constants $c_2'''$ and $C_3'''$ such that with probability at least $1 - \exp(-c_2''' n)$,
    \[
    \| R(\w) \|_\infty \;\le\; C_3'''\,\tau\;\sqrt{\frac{n}{N}}
    \]
    holds simultaneously for all $\w \in \mathbb{Q}_1$ with $\w \neq \pm \w^*$.
\end{lemma}

\begin{proof}[Proof of Lemma~\ref{lem:Concentration_of_main_term_3}.]
Let $\w, \w^* \in \mathbb{Q}_1$ with $\w \ne \pm \w^*$. Define the scalar random variable for each coordinate $j \in \{1, \dots, n\}$:
\[ S_{i,j}(\w) = D(\w;\z^{(i)}) \1_{\bigl\{\z^{(i)\top} \w>0\bigr\}} \inner{\b_i(\w), \e_j}. \]
The $j$-th component of the vector $R(\w)$ is $R_j(\w) = \frac{\tau}{N} \sum_{i=1}^N S_{i,j}(\w)$.

First, we compute the expectation of $S_{i,j}$.  We know that $$
\mathbb{E}_{\Z} \left[D(\w;\z^{(i)}) \1_{\bigl\{\z^{(i)\top} \w>0\bigr\}} \,\z^{(i)}\right] = \frac{1}{2\sqrt{2\pi}}(\w-\w^*).
$$
Since $\b_i = \z^{(i)} - \inner{\z^{(i)}, \u_1}\u_1 - \inner{\z^{(i)}, \u_2}\u_2$, by linearity we have:
\small
\begin{align*}
\mathbb{E}_{\Z} [S_{i,j}] &= \mathbb{E}_{\Z} \left[ D(\w;\z^{(i)}) \1_{\bigl\{\z^{(i)\top} \w>0\bigr\}} \inner{\b_i(\w), \e_j} \right] \\
&= \mathbb{E}_{\Z} \left[ (D)\1 \left( \inner{\z^{(i)}, \e_j} - \inner{\z^{(i)}, \u_1}\inner{\u_1, \e_j} - \inner{\z^{(i)}, \u_2}\inner{\u_2, \e_j} \right) \right] \\
&= \inner{\mathbb{E}_{\Z}[(D)\1\z^{(i)}], \e_j} - \inner{\u_1, \e_j} \mathbb{E}_{\Z}[(D)\1\inner{\z^{(i)}, \u_1}] - \inner{\u_2, \e_j} \mathbb{E}_{\Z} [(D)\1\inner{\z^{(i)}, \u_2}] \\
&= \inner{\frac{1}{2\sqrt{2\pi}}(\w-\w^*), \e_j} - \inner{\u_1, \e_j} \mathbb{E}_{\Z}[X_i(\w))] - \inner{\u_2, \e_j} \mathbb{E}_{\Z}[Y_i(\w)] \\
&= \frac{1}{2\sqrt{2\pi}}(w_j-w_j^*) - \inner{\u_1, \e_j} \frac{\|\w-\w^*\|}{2\sqrt{2\pi}} - \inner{\u_2, \e_j} \cdot 0 \\
&= \frac{1}{2\sqrt{2\pi}}(w_j-w_j^*) - \frac{w_j-w^*_j}{\|\w-\w^*\|} \frac{\|\w-\w^*\|}{2\sqrt{2\pi}} = 0.
\end{align*}
\normalsize

Next, we establish that $S_{i,j}$ is sub-Gaussian. We have $|S_{i,j}| \le |\inner{\b_i, \e_j}|$. Since $\b_i$ is the projection of $\z^{(i)} \sim \mathcal{N}(0, I_n)$ onto the subspace orthogonal to $\u_1$ and $\u_2$, $\b_i$ is a mean-zero Gaussian vector within that subspace. The term $\inner{\b_i, \e_j}$ is the projection of this Gaussian vector onto $\e_j$, which makes $\inner{\b_i, \e_j}$ a mean-zero Gaussian random variable with variance $\le 1$. Since $|S_{i,j}| \le |\inner{\b_i, \e_j}|$ and $\inner{\b_i, \e_j}$ is Gaussian with variance $\le 1$ (hence sub-Gaussian), $S_{i,j}$ is also sub-Gaussian with a uniformly bounded sub-Gaussian norm $\|S_{i,j}\|_{\psi_2} \le C'$ for some universal constant $C'> 0$.

Because $S_{i,j}(\w)$ are independent (across $i$), mean-zero, sub-Gaussian random variables with uniformly bounded sub-Gaussian norm $C'$, we can apply Hoeffding's bound to their sum for fixed $j, \w$:
\[
\Pr\left(\left|\frac{1}{N}\sum_{i=1}^N S_{i,j} \right|>t\right)
\;\le\;
2\exp\bigl(-c\,N\,t^2 / (C')^2 \bigr).
\]
Next, we obtain a uniform bound over $j\in[n]$ and all $\w \in \mathbb{Q}_1\setminus \{\pm \w^*\}$, which are over fewer than $M = n \cdot 2^{n}$ events. We choose $t$ such that the failure probability for a single event is $\le \delta/M$, where $\delta = \exp(-c_2''' n)$ is the overall target failure probability. This requires $t^2 \gtrsim \frac{(C')^2}{cN} \ln(M/\delta) \approx \frac{n}{N}$. We set $t = 2 \frac{(C')^2}{cN} \ln(M/\delta) = C''\sqrt{n/N}$ for a sufficiently large constant $C''$.
By the union bound, with probability at least $1 - \delta = 1 - \exp(-c_2''' n)$:
\[ \left| \frac{1}{N} \sum_{i=1}^N S_{i,j} \right| \le t = C''\sqrt{\frac{n}{N}} \]
holds simultaneously for all $j\in[n]$ and all $\w \in \mathbb{Q}_1\setminus \{\pm \w^*\}$.
Multiplying by $\tau$ and taking the maximum over $j$:
\[ \| R(\w) \|_\infty = \max_{j\in[n]} \left| \frac{\tau}{N} \sum_{i=1}^N S_{i,j} \right| \le C'' \tau \sqrt{\frac{n}{N}}. \]
Setting $C_3''' = C''$ yields the claimed bound.
\end{proof}

\subsection{Concentration Bound for the Cross Term}\label{sec:corss_term}
Define the cross-term vector $T_{\mathrm{cross}}(\w) \in \mathbb{R}^n$ whose $p$-th component is given by
    \[ [T_{\mathrm{cross}}(\w)]_p = \sum_{j=1}^m \sum_{\substack{k=1\\k\neq j}}^m v_j\,v_k\, \frac{1}{N} \sum_{i=1}^N \Bigl(\1_{\{\z_j^{(i)\top} \w>0\}}\,[\z_j^{(i)}]_p\Bigr)\; \Bigl(\1_{\{\z_k^{(i)\top} \w>0\}} - \1_{\{\z_k^{(i)\top} \w^*>0\}} \Bigr). \]

\begin{lemma}[ $\ell_\infty$ Concentration of Cross Terms]
\label{lem:Concentration_of_cross_term_linf_alt}
Suppose that Assumptions~\ref{assump:data}, \ref{assump:noise}, and~\ref{assump:sample_size} hold. Let $\w^* \in \mathbb{Q}_1$  and $\tau = 2\sqrt{2\pi}$. Then, there exist positive universal constants $c_5$ and $C_3$ such that with probability at least $1 - \exp(-c_5 n)$,
\[ \| \tau T_{\mathrm{cross}}(\w) \|_\infty \le C_3\,\tau\,\vnorm{\v}_1^2 \sqrt{\frac{n + 2\ln m}{N}} \] holds simultaneously for all $\w \in \mathbb{Q}_1$.
\end{lemma}

\begin{proof}[Proof of Lemma \ref{lem:Concentration_of_cross_term_linf_alt}. ]
Fix $\w \in \mathbb{Q}_1$, indices $j \neq k$ from $\{1, \dots, m\}$, and a coordinate $p \in \{1, \dots, n\}$.
Let $X_{j,p}^{(i)}(\w) = \1_{\{\z_j^{(i)\top} \w>0\}}\,[\z_j^{(i)}]_p$ and $\Delta_k^{(i)}(\w) = \1_{{\{\z_k^{(i)\top} \w>0\}}} - \1_{\{\z_k^{(i)\top} \w^*>0\}}$.
Define the single term contribution from sample $i$ as $A_{ijk,p} = X_{j,p}^{(i)}(\w) \Delta_k^{(i)}(\w)$, and its average over $N$ samples as $S_{jk, p} = \frac{1}{N} \sum_{i=1}^N A_{ijk, p}$.
The $p$-th component of the cross-term vector is $[T_{\mathrm{cross}}(\w)]_p = \sum_{j \neq k} v_j v_k S_{jk, p}$, where $\sum_{j \neq k}$ is the double summation over the indices $j,k \in [m]$ with $j \neq k$. We want to bound $\|T_{\mathrm{cross}}(\w) \|_\infty$.

\textbf{Step 1: Expectation.}
Since $j \neq k$, $\z_j^{(i)}$ and $\z_k^{(i)}$ are independent random vectors. Therefore, $X_{j,p}^{(i)}(\w)$ (which depends only on $\z_j^{(i)}$) and $\Delta_k^{(i)}(\w)$ (which depends only on $\z_k^{(i)}$) are independent random variables for fixed $\w,  p$. We also know from  the rotational invariance of Gaussian measure that $\mathbb{E}[\Delta_k^{(i)}(\w)] = \mathbb{E}[\1_{\{\z_k^{(i)\top} \w>0\}}] - \mathbb{E}[\1_{\{\z_k^{(i)\top} \w^*>0\}}] = 1/2 - 1/2 = 0$.
Thus, the expectation of the single term is $\mathbb{E}[A_{ijk,p}] = \mathbb{E}[X_{j,p}^{(i)}(\w)] \mathbb{E}[\Delta_k^{(i)}(\w)] = 0$.
By linearity of expectation, the expectation of the average is $\mathbb{E}[S_{jk, p}] = \frac{1}{N} \sum_{i=1}^N \mathbb{E}[A_{ijk,p}] = 0$.

\textbf{Step 2: sub-Gaussianity of Individual Terms.}
The term $\left|X_{j,p}^{(i)}(\w)\right| \le \left|[\z_j^{(i)}]_p\right|$, and $[\z_j^{(i)}]_p \sim \mathcal{N}(0,1)$ is a standard Gaussian, which is sub-Gaussian. Thus, $X_{j,p}^{(i)}(\w)$ is sub-Gaussian with a $\psi_2$-norm bounded by an universal constant, say $\|X_{j,p}^{(i)}\|_{\psi_2} \le C'$.
The term $\Delta_k^{(i)}(\w)$ is bounded: $|\Delta_k^{(i)}(\w)| \le 1$.
The product $A_{ijk,p} = X_{j,p}^{(i)} \Delta_k^{(i)}$ is the product of a sub-Gaussian variable and a bounded variable, hence it is sub-Gaussian. Its norm satisfies $\|A_{ijk,p}\|_{\psi_2} \le \|\Delta_k^{(i)}\|_\infty \|X_{j,p}^{(i)}\|_{\psi_2} \le 1 \cdot C'$ for some universal constant $C' > 0$, uniformly over all $i,j,k,p$ and $\w \in \mathbb{Q}_1$.
The terms $A_{ijk,p}$ are independent across the sample index $i$ because the samples $(\z_1^{(i)}, \dots, \z_m^{(i)})$ are independent across $i$.

\textbf{Step 3: Uniform Concentration for $S_{jk,p}$.}
Since $S_{jk, p}$ is the average of $N$ independent, mean-zero, sub-Gaussian variables $A_{ijk,p}$ (with $\psi_2$-norm $\le C'$), we can apply Hoeffding's bound (\cite{vershynin2018high}): for a fixed set of indices $(j, k, p)$ and fixed vectors $\w$,
\[ \Pr\left(\left| S_{jk, p} \right|>t\right) \le 2\exp\bigl(-c\,N\,t^2 / (C')^2 \bigr) \]
for some universal constant $c>0$.
To ensure this holds uniformly over all possibilities, we apply a union bound. The number of possibilities is:
\begin{itemize}
    \item Pairs $(j, k)$ with $j \neq k$: $m(m-1)$ pairs.
    \item Coordinate $p \in \{1, \dots, n\}$: $n$ coordinates.
    \item The vectors $\w$ from $\mathbb{Q}_1$: At most $|\mathbb{Q}_1| = 2^{n}$.
\end{itemize}
The total number of events to control is $M' = m(m-1) \cdot n \cdot 2^{n}$.
Let $\delta = \exp(-c_5 n)$ be the target overall failure probability (for some chosen constant $c_5 > 0$). The required failure probability per event is $\delta' = \delta / M'$.
We need to choose $t$ such that $2\exp(-c N t^2 / (C')^2) \le \delta'$. This requires
\[ \frac{c N t^2}{(C')^2} \ge \ln(2/\delta') = \ln(2 M' / \delta) \approx \ln(m^2 n) + n \ln 2 + c_5 n. \]
Choosing $t$ to satisfy the inequality yields:
\[ t^2 = 2 \frac{(C')^2}{cN} ( (2\ln 2 + c_5)n + \ln(m^2 n) ) \]
To make the notation light, we can choose a constant $C_{\mathrm{alt}}$ such that setting
\[ t = C_{\mathrm{alt}} \sqrt{\frac{n + 2\ln m}{N}} \]
(absorbing constants $C', c, c_5$ terms into $C_{\mathrm{alt}}$) satisfies the requirement for a sufficiently large $N$.
By the union bound, with probability at least $1 - \delta = 1 - \exp(-c_5 n)$, the bound
\[ |S_{jk, p}| \le C_{\mathrm{alt}} \sqrt{\frac{n + 2\ln m}{N}} \]
holds simultaneously for all $j \neq k$, all $p$, and all vectors $\w \in \mathbb{Q}_1$.

\textbf{Step 4: Bounding the Sum}
Now we bound the target quantity 

$[T_{\mathrm{cross}}(\w)]_p = \sum_{j \neq k} v_j v_k S_{jk, p}$ using the uniform bound derived above via triangle inequality:
\begin{align*} \left| [T_{\mathrm{cross}}(\w)]_p \right| &\le \sum_{j \neq k} |v_j v_k| |S_{jk, p}| \\ &\le \sum_{j \neq k} |v_j v_k| \left( C_{alt} \sqrt{\frac{n + 2\ln m}{N}} \right) \\ &= \left( \sum_{j \neq k} |v_j v_k| \right) C_{alt} \sqrt{\frac{n + 2\ln m}{N}} \end{align*}
Using the inequality $\sum_{j \neq k} |v_j v_k| \le (\sum_j |v_j|) (\sum_k |v_k|) = \vnorm{\v}_1^2$, we get:
\[ \left| [T_{\mathrm{cross}}(\w)]_p \right| \le C_{\mathrm{alt}} \vnorm{\v}_1^2 \sqrt{\frac{n + 2\ln m}{N}} \]
This holds simultaneously for all $p$ and relevant $(\w)$ with probability at least $1-\delta$.

Finally, taking the maximum over $p \in \{1, \dots, n\}$ and multiplying by $\tau = 2\sqrt{2\pi}$:
\[ \| \tau T_{\mathrm{cross}}(\w) \|_\infty = \max_{p} |\tau [T_{\mathrm{cross}}(\w)]_p| \le C_{\mathrm{alt}}\, \tau\, \vnorm{\v}_1^2 \sqrt{\frac{n + 2\ln m}{N}} \]
Defining the constant $C_3 = C_{\mathrm{alt}}$ (which is independent of $\v, n, N$) yields the stated result.
\end{proof}

\subsection{Concentration Bound for the Noise Term}\label{sec:noise_term}
 Define the noise contribution vector $T_{noise}(\w) \in \mathbb{R}^n$ for any $\w \in \mathbb{Q}_1$ whose components are given by
    \[ [T_{\mathrm{noise}}(\w)]_p = -\frac{1}{N} \sum_{i=1}^N \sum_{j=1}^m v_j  \left( \1_{\{\z_j^{(i)\top} \w > 0\}} [\z_j^{(i)}]_p \right) \xi^{(i)}. \]
\begin{lemma}[Uniform $\ell_\infty$ Concentration of Noise Term]
\label{lem:Concentration_of_noise_term_uniform}
Suppose that Assumptions~\ref{assump:data}, \ref{assump:noise}, and~\ref{assump:sample_size} hold. Let $\w^* \in \mathbb{Q}_1$  and $\tau = 2\sqrt{2\pi}$. Then, there exist positive universal constants
  $c_6$ and $C_8$ such that with probability at least $1 - \exp(-c_6 n)$,
    \[  \| \tau T_{\mathrm{noise}}(\w) \|_\infty  \le C_8 \tau K_\xi \|\v\|_1 \sqrt{\frac{n + 2\ln m}{N}}, \]
    holds uniformly for all $\w \in \mathbb{Q}_1$.
\end{lemma}

\begin{proof}[Proof of of Lemma \ref{lem:Concentration_of_noise_term_uniform}.]
For fixed vectors $\w \in \mathbb{Q}_1$, indices $j \in \{1, \dots, m\}$, $p \in \{1, \dots, n\}$, and sample index $i \in \{1, \dots, N\}$, define:
\begin{itemize}
    \item $X_{j,p}^{(i)}(\w) = \1_{\{\z_j^{(i)\top} \w > 0\}} [\z_j^{(i)}]_p$.
    \item $A_{ij,p}(\w) = X_{j,p}^{(i)}(\w) \xi^{(i)}$.
    \item $S_{j, p}(\w) = \frac{1}{N} \sum_{i=1}^N A_{ij,p}(\w)$.
\end{itemize}
The $p$-th component of the noise term for a given $\w$ is $[T_{\mathrm{noise}}(\w)]_p = - \sum_{j=1}^m  v_j S_{j, p}(\w)$. We want to bound its magnitude uniformly over $p$ and $\w \in \mathbb{Q}_1$.

\textbf{Step 1: Expectation.}
For any fixed $\w \in \mathbb{Q}_1$, since $\z_j^{(i)}$ and $\xi^{(i)}$ are independent, and $\mathbb{E}[\xi^{(i)}] = 0$ from Assumption~\ref{assump:noise}, we have:
\[ \mathbb{E}[A_{ij,p}(\w)] = \mathbb{E}[X_{j,p}^{(i)}(\w) \xi^{(i)}] = \mathbb{E}[X_{j,p}^{(i)}(\w)] \mathbb{E}[\xi^{(i)}] = 0. \]
By linearity, $\mathbb{E}[S_{j, p}(\w)] = 0$ for any fixed $\w$.

\textbf{Step 2: Properties of $A_{ij,p}(\w)$ for Bernstein's Inequality.}
The term $|X_{j,p}^{(i)}(\w)| \le |[\z_j^{(i)}]_p|$. Since $[\z_j^{(i)}]_p$ is a standard Gaussian component, it is sub-Gaussian with a uniform sub-Gaussian norm, so $\|[\z_j^{(i)}]_p\|_{\psi_2} \le C'$. Thus, $X_{j,p}^{(i)}(\w)$ is also sub-Gaussian with $\|X_{j,p}^{(i)}(\w)\|_{\psi_2} \le C'$ uniformly over all $i,j,p$ and $\w \in \mathbb{Q}_1$.
The noise term $\xi^{(i)}$ is sub-Gaussian with $\|\xi^{(i)}\|_{\psi_2} \le K_\xi$ by Assumption~\ref{assump:noise}.
The term $A_{ij,p}(\w) = X_{j,p}^{(i)}(\w) \xi^{(i)}$ is the product of two independent, mean-zero sub-Gaussian random variables. Therefore, $A_{ij,p}(\w)$ is a mean-zero sub-exponential random variable and its sub-exponential norm is bounded by $K_A := C' K_\xi$ by Lemma 2.8.6 in \cite{vershynin2018high}.
The variables $A_{ij,p}(\w)$ are independent across the sample index $i$ for any fixed $j,p,\w$.

\textbf{Step 3: Uniform Concentration for $S_{j, p}(\w)$.}
$S_{j, p}(\w)$ is the average of $N$ independent, mean-zero, sub-exponential variables $A_{ij,p}(\w)$. Applying Bernstein's inequality \cite{vershynin2018high} for sub-exponential random variables: for a fixed $\w, j, p$,
\[ \Pr\left(\left| S_{j, p}(\w) \right|>t\right) \le 2\exp\left(-c N \min\left(\frac{t^2}{K_A^2}, \frac{t}{K_A}\right) \right) \]
for some universal constant $c>0$.
Now, we want that the above holds uniformly over $j \in \{1,\dots,m\}$, $p \in \{1,\dots,n\}$, and $\w \in \mathbb{Q}_1$.
The total number of $(j,p,\w)$ tuples to control is $M' = m \cdot n \cdot |\mathbb{Q}_1| = m n 2^n$.
Let $\delta = \exp(-c_6 n)$ be the target overall failure probability (for some chosen $c_6 > 0$). The required failure probability per event is $\delta' = \delta / M'$.
We choose $t$ such that $2\exp\left(-c N \min\left(\frac{t^2}{K_A^2}, \frac{t}{K_A}\right) \right) \le \delta'$. This requires:
\[ c N \min\left(\frac{t^2}{K_A^2}, \frac{t}{K_A}\right) \ge \ln(2/\delta') = \ln(2 M' / \delta) \approx \ln(m n 2^n) + c_6 n = (\ln 2 + c_6) n + \ln(mn). \]

Let $R_{n,m} = (\ln 2 + c_6)n + \ln(2mn)$. We aim to find $t$ satisfying
\[ c N \min\left(\frac{t^2}{K_A^2}, \frac{t}{K_A}\right) \ge R_{n,m} \quad (*). \]
We select $t = K_A \sqrt{\frac{R_{n,m}}{c N}}$. This choice corresponds to the quadratic term $\frac{t^2}{K_A^2}$ being active in $(*)$, provided $t \le K_A$, which is equivalent to $N \ge \frac{R_{n,m}}{c}$.
By Assumption~\ref{assump:sample_size}, $N \ge C_1 n^2$ for a sufficiently large universal constant $C_1$. Since $R_{n,m} = O(n + \ln m)$ (From Assumption~\ref{assump:data}, we have $m=O(2^n)$ implying $R_{n,m}=O(n)$), we choose $C_1$ large enough to ensure such that $C_1 n^2 \ge \frac{2R_{n,m}}{c}$ for all $n \ge 1$.
Thus, under Assumption~\ref{assump:sample_size}, our choice of $t$ derived from the quadratic term is valid.

Therefore, we take:
\[ t= K_A \sqrt{\frac{R_{n,m}}{c N}} = C' K_\xi \sqrt{\frac{(\ln 2 + c_6)n + \ln(2mn)}{c N}}. \]
As $R_{n,m} = O(n + \ln m)$, we define $C_8$ as an appropriate universal constant incorporating all numerical factors, such that:
\[ t= C_8 K_\xi \sqrt{\frac{n + 2\ln m}{N}}. \]

By the union bound, with probability at least $1 - \delta = 1 - \exp(-c_6 n)$, the bound
\[ |S_{j, p}(\w)| \le t\]
holds simultaneously for all $j \in \{1,\dots,m\}$, $p \in \{1,\dots,n\}$, and all $\w \in \mathbb{Q}_1$.

\textbf{Step 4: Bounding the Sum.}
Now we bound the target quantity $[T_{\mathrm{noise}}(\w)]_p = - \sum_{j=1}^m v_j S_{j, p}(\w)$ using the uniform bound on $|S_{j, p}(\w)|$:
\begin{align*} \left| [T_{\mathrm{noise}}(\w)]_p \right| &= \left| \sum_{j=1}^m  v_j  S_{j, p}(\w) \right| \\ &\le \sum_{j=1}^m  |v_j | |S_{j, p}(\w)| \\ &\le \sum_{j=1}^m  |v_j | \left( C_8 K_\xi \sqrt{\frac{n + 2\ln m}{N}} \right).
\end{align*}
Thus, we have
\[ \left| [T_{\mathrm{noise}}(\w)]_p \right| \le C_8 K_\xi \|\v\|_1 \sqrt{\frac{n + 2\ln m}{N}}. \]
This holds simultaneously for all $p$ and all $\w \in \mathbb{Q}_1$ with probability at least $1-\delta$. Therefore, we have
\[ \max_{\w \in \mathbb{Q}_1} \|\tau T_{\mathrm{noise}}(\w) \|_\infty = \max_{\w \in \mathbb{Q}_1} \max_{p\in[n]} |\tau[T_{\mathrm{noise}}(\w)]_p| \le C_8 \tau K_\xi \|\v\|_1 \sqrt{\frac{n + 2\ln m}{N}}. \]
\end{proof}

The short combination argument proving Theorem~\ref{thm:RAIC_for_two_layer_quantized_neural_networks_appendix} is included in Section~\ref{sec:concentration}.

\section{Ergodic Convergence Analysis}
\label{sec:Drift_analysis_componentwise_dynamical_system_for_Ergodic_convergence_appendix}
The representative constant-step-size proof of Theorem~\ref{thm:Informal_ergodic_informal_nonasymp}, which gives the core occupation-time mechanism, is included in Section~\ref{sec:ergodic_constant_proof_main} to motivate the argument. Accordingly, this appendix section points to the remaining extension to general step-size schedules satisfying Assumption~\ref{assump:step_size} and to nonzero initialization, which is given in Section~\ref{sec:general_decaying_step_size_analysis_final_corrected_appendix}.

\section{Non-Ergodic (Last-Iterate) Convergence Analysis}
\label{sec:Non_Ergodic_Behavior_appendix}
As before, we start with the constant step size case and the zero initialization for the proof of  Theorem~\ref{thm:global_coordinate_recovery_noisy} in this section. We complete the full proof in Section~\ref{sec:general_decaying_step_size_analysis_final_corrected_appendix} by extending to the general step-size scheduling satisfying Assumption~\ref{assump:step_size} and by relaxing the zero initialization.

\begin{theorem}[Recurrence and Instability]
\label{thm:noisy_recurrence_departure}
Suppose that Assumptions~\ref{assump:data} and~\ref{assump:noise} hold.
Assume that the sample size $N$ is large enough such that
\[
N \ge \max\left( C_1 n^2, \left(  \frac{5 (n+4)^2 (C' \|\v\|_1^2 + C''K_\xi \|\v\|_1)}{2\|\v\|^2} \right)^2  \right),
\]
for some positive universal constants $C_1$, $C'$ and $C''$.
Let $\|v\|_0:=|\operatorname{supp}(\v)|=|\{j\in[m]:v_j\neq0\}|$.

Then, with probability at least $1 - 5m\exp(-cn)$ for some positive universal constant $c > 0$, the iterates $\w^t$ generated by the STE-gradient method with a constant step size $\eta_0$ and $\x^0 = 0$ visit the target state $\w^*$ infinitely often.
Furthermore, for the noisy observation case under the additional assumption that the noise variables $\xi^{(i)}$ are independent of the data and across samples, continuous, and not a.s. zero, the iterates also depart from $\w^*$ infinitely often with probability at least $ \frac{1}{2}-2^{-\|v\|_0N-1}-5m\exp(-cn)$,

over the draw of the noisy samples $\{(\Z^{(i)},\xi^{(i)})\}_{i=1}^N$ that determines $\tilde{\nabla}_{\w}L(\w^*)$.
\end{theorem}

\begin{proof}[Proof of Theorem~\ref{thm:noisy_recurrence_departure}.]
Assume that the conditions of Theorem~\ref{thm:ergodic_informal_nonasymp} hold.
Recall that
\[
    \rho = \frac{(C'\|\v\|_1^2 +  C''K_\xi\|\v\|_1) n}{2\|\v\|^2 \sqrt{N}}
    \quad\text{and}\quad
    L \;=\; \Bigl\lceil \tfrac{\,1-2\rho\,}{\,\rho\,}\Bigr\rceil.
\]
A simple algebraic computation shows that the sample complexity requirement in Theorem~\ref{thm:noisy_recurrence_departure} implies $\rho < 1/(n+1)$, which in turn implies $L+1>n$.

The cycle dynamics, established in Steps 1--2 of the proof of Theorem~\ref{thm:ergodic_informal_nonasymp}, govern each coordinate $x_j^t$. This ensures incorrect signs occur at most once per cycle of minimum length $L+1 = 1+\lceil(1-2\rho)/\rho\rceil$.

\bigskip

\textbf{Infinite Visits to $\w^*$.}
Let $S_j(T)$ be the set of times up to $T$ such that $w_j^t\ne w_j^*$ for a given index $j$. From Step 2 of the proof of Theorem~\ref{thm:ergodic_informal_nonasymp}, we have shown $|S_j(T)|\le \lceil T/(L+1)\rceil$. Define $U_T=\bigcup_{j=1}^n S_j(T)$ as the set of times up to $T$ where $\w^t\ne\w^*$. Thus, by the union bound,
\[
    |U_T| \le \sum_{j=1}^n |S_j(T)| \le n\lceil T/(L+1)\rceil.
\]
Therefore,
\[
    |U_T| \le n(T/(L+1)+1)=\frac{n}{L+1}T+n.
\]
Let $G_T=\{t\in\{1,\dots,T\}:\w^t=\w^*\}$ be the set of times where the iterate hits the target. The number of visits up to time $T$ is $|G_T|=T-|U_T|$, so
\[
|G_T|\ge T-\left(\frac{n}{L+1}T+n\right)=\left(1-\frac{n}{L+1}\right)T-n.
\]
Since $L+1>n$, the factor $c_0:=1-n/(L+1)$ is positive. Hence $|G_T|\ge c_0T-n\to\infty$ as $T\to\infty$, and the iterates visit $\w^*$ infinitely often.

\bigskip

\textbf{Infinite Departures from $\w^*$.}
We now show that, under continuous nonzero observation noise, the iterates cannot remain permanently at $\w^*$ after visiting it.
We first use the following lemma.

\begin{lemma}[Symmetry and Non-vanishing Probability of the Gradient with High Probability]
\label{lem:gradient_symmetry_nonvanishing}
Let
\[
\|\v\|_0:=|\operatorname{supp}(\v)|=|\{j\in[m]:v_j\neq0\}|.
\]
Define the random vector $\boldsymbol{G}\in\mathbb{R}^n$ by $\boldsymbol{G}=-\tilde{\nabla}_{\w}L(\w^*)$, with $p$-th component
\[
G_p
=
\frac{1}{N}\sum_{i=1}^N
\left(
\sum_{j'=1}^m v_{j'}\1_{\{\z_{j'}^{(i)\top}\w^*>0\}}[\z_{j'}^{(i)}]_p
\right)\xi^{(i)}.
\]
Under Assumptions~\ref{assump:data} and~\ref{assump:noise}, and assuming in addition that the noise variables $\xi^{(i)}$ are independent of the data and across samples, continuous, and not a.s. zero, we have:
\begin{enumerate}
    \item The vector $\boldsymbol{G}$ is symmetric about zero, i.e., $\boldsymbol{G}\stackrel{d}{=}-\boldsymbol{G}$. In particular, each component $G_p$ is symmetric about zero.
    \item For each $p\in[n]$,
    \[
    \mathbb{P}(G_p=0)=2^{-\|\v\|_0N},
    \qquad
    \mathbb{P}(G_p>0)=\mathbb{P}(G_p<0)
    =\frac{1}{2}- 2^{-\|\v\|_0N-1}.
    \]
    Moreover,
    \[
    \mathbb{P}(\boldsymbol{G}=\mathbf{0})=2^{-\|\v\|_0N}.
    \]
\end{enumerate}
\end{lemma}

\begin{proof}[Proof of Lemma~\ref{lem:gradient_symmetry_nonvanishing}.]

\textbf{Vector symmetry of $\boldsymbol{G}$.}
Condition on the data $\{\Z^{(i)}\}_{i=1}^N$ and define
\[
\boldsymbol{C}_i:=\sum_{j'=1}^m v_{j'}\1_{\{\z_{j'}^{(i)\top}\w^*>0\}}\z_{j'}^{(i)}.
\]
Then
\[
\boldsymbol{G}=\frac1N\sum_{i=1}^N \boldsymbol{C}_i\xi^{(i)}.
\]
Since the noise variables are symmetric and independent across samples, the noise vector $(\xi^{(1)},\dots,\xi^{(N)})$ is jointly symmetric; moreover, it is independent of the data. Thus the noise vector has the same distribution as its negative. Hence, conditional on the data and therefore unconditionally, $\boldsymbol{G}\stackrel d= -\boldsymbol{G}$.

\textbf{Symmetry of $G_p$:}
    Let $G_p = \frac{1}{N} \sum_{i=1}^N Y_i$, where
    \[
    Y_i = C_p(\z^{(i)}) \xi^{(i)}
    \quad \text{and} \quad
    C_p(\z^{(i)}) =
    \sum_{j'=1}^m
    v_{j'} \1_{\{\z_{j'}^{(i)\top} \w^* > 0\}} [\z_{j'}^{(i)}]_p .
    \]
    The terms $Y_i$ are i.i.d. because each pair $(\z^{(i)}, \xi^{(i)})$ is i.i.d. We show that each $Y_i$ is symmetric about zero using its characteristic function, $\phi_{Y_i}(s) = \mathbb{E}[e^{isY_i}]$. This can be expressed as a nested expectation with the conditional expectation on $\z^{(i)}$ and thus on the value $c = C_p(\z^{(i)})$:
    \[
    \phi_{Y_i}(s)
    =
    \mathbb{E}_{\z^{(i)}}
    \left[
    \mathbb{E}_{\xi^{(i)}} [e^{isc \xi^{(i)}} \mid \z^{(i)}]
    \right]
    =
    \mathbb{E}_{\z^{(i)}} [ \phi_{\xi^{(i)}}(sc) ] .
    \]
    Since $\xi^{(i)}$ is symmetric about zero, its characteristic function $\phi_{\xi^{(i)}}(u)$ is real-valued and even, i.e.,
    $\phi_{\xi^{(i)}}(u)=\phi_{\xi^{(i)}}(-u)$.
    Thus, $\phi_{\xi^{(i)}}(sc)$ is real and even in $s$. Then, $\phi_{Y_i}(s)$ is real. Also,
    \[
    \phi_{Y_i}(-s)
    =
    \mathbb{E}_{\z^{(i)}} [ \phi_{\xi^{(i)}}(-s C_p(\z^{(i)})) ]
    =
    \mathbb{E}_{\z^{(i)}} [ \phi_{\xi^{(i)}}(s C_p(\z^{(i)})) ]
    =
    \phi_{Y_i}(s).
    \]
    Since $\phi_{Y_i}(s)$ is real and even, $Y_i$ is symmetric about zero. The sum of i.i.d. symmetric random variables, $\sum_i Y_i$, is symmetric, and so is $G_p$.

\textbf{Sign probabilities of $G_p$.}
Let
\[
\|\v\|_0:=|\operatorname{supp}(\v)|=|\{j\in[m]:v_j\neq0\}|.
\]
For fixed $i$ and $p$, the coefficient
\[
C_p(\z^{(i)})
=
\sum_{j'=1}^m
v_{j'} \1_{\{\z_{j'}^{(i)\top} \w^* > 0\}} [\z_{j'}^{(i)}]_p
\]
equals zero whenever all indicator terms corresponding to $j'\in\operatorname{supp}(\v)$ are inactive. This event has probability $2^{-\|\v\|_0}$. On the complementary event, at least one relevant indicator function is active, and $C_p(\z^{(i)})$ is a nontrivial continuous random variable; hence it is nonzero almost surely. Therefore,
\[
P(C_p(\z^{(i)})=0)=2^{-\|\v\|_0}.
\]
By independence across samples,
\[
P(C_p(\z^{(i)})=0\ \text{for all } i=1,\dots,N)
=
2^{-\|\v\|_0N}.
\]
On the complementary event, at least one coefficient $C_p(\z^{(i)})$ is nonzero. Conditional on the data and on all other noise variables, $G_p$ is then an affine function of one continuous noise variable with nonzero slope. Hence the conditional probability that $G_p=0$ is zero. Consequently,
\[
P(G_p=0)=2^{-\|\v\|_0N}.
\]
Since $G_p$ is symmetric about zero,
\[
P(G_p>0)=P(G_p<0)
=
\frac{1-P(G_p=0)}{2}
=
\frac{1}{2}- 2^{-\|\v\|_0N-1}.
\]
This also gives
\[
P(G_p\neq0)=1-2^{-\|\v\|_0N}.
\]

Finally, define the vector coefficients $
\boldsymbol{C}_i:=\sum_{j'=1}^m v_{j'}\1_{\{\z_{j'}^{(i)\top}\w^*>0\}}\z_{j'}^{(i)}$.
The same argument gives $\boldsymbol{C}_i=\mathbf0$ with probability $2^{-\|\v\|_0}$, and $\boldsymbol{C}_i\ne\mathbf0$ a.s. conditional on a nonempty active subset. If at least one $\boldsymbol{C}_i\ne\mathbf0$, then conditional on the data and all other noises, the vector $\boldsymbol{G}=N^{-1}\sum_i \boldsymbol{C}_i\xi^{(i)}$ cannot equal $\mathbf0$ except on a probability-zero event. Therefore
\[
\mathbb{P}(\boldsymbol{G}=\mathbf0)=2^{-\|\v\|_0N}.
\]
\end{proof}

\paragraph{Proof of Departure by Contradiction.}
Fix any coordinate $j\in[n]$. Recall that the update rule of the STE-gradient method is
\[
\x^u = \x^{u-1} - \eta_0\tilde{\nabla}_{\w} L(\w^{u-1}).
\]
Suppose, for contradiction, that the iteration gets stuck at $\w^*$ from some time $t$ onwards, i.e., $\w^u=\w^*$ for all $u\ge t$.

Since $\boldsymbol{G}=-\tilde{\nabla}_{\w}L(\w^*)$, Lemma~\ref{lem:gradient_symmetry_nonvanishing} implies that
\[
P\left(
w_j^*[\tilde{\nabla}_{\w}L(\w^*)]_j>0
\right)
=
\frac{1}{2}- 2^{-\|\v\|_0N-1}.
\]
On this event, the iterate must eventually depart from $\w^*$.

Indeed, if $w_j^*=1/\sqrt n$, then
$[\tilde{\nabla}_{\w}L(\w^*)]_j>0$. While $\w^u=\w^*$, we have $x_j^u\ge0$, and the update
\[
x_j^{u+1}=x_j^u-\eta_0[\tilde{\nabla}_{\w}L(\w^*)]_j
\]
shows that $x_j^u$ strictly decreases at every step. Hence $x_j^u$ must eventually become negative, contradicting $\w^u=\w^*$.

Similarly, if $w_j^*=-1/\sqrt n$, then
$[\tilde{\nabla}_{\w}L(\w^*)]_j<0$. While $\w^u=\w^*$, we have $x_j^u<0$, and the update
\[
x_j^{u+1}=x_j^u-\eta_0[\tilde{\nabla}_{\w}L(\w^*)]_j
\]
shows that $x_j^u$ strictly increases at every step. Hence $x_j^u$ must eventually become nonnegative, again contradicting $\w^u=\w^*$.

Therefore, the event ensuring departure occurs with probability at least
\[
\frac{1}{2}- 2^{-\|\v\|_0N-1}.
\]
Thus, combining the previous infinite-visit result with this departure event, we show that the iterates visit $\w^*$ infinitely often and also depart from $\w^*$ infinitely often with probability at least
\[
\frac{1}{2}- 2^{-\|\v\|_0N-1}-5m\exp(-cn),
\]
where the term $5m\exp(-cn)$ comes from the failure probability of the finite-sample approximation bound in Theorem~\ref{thm:RAIC_for_two_layer_quantized_neural_networks_appendix}. This confirms the recurrent, non-stable behavior stated in the theorem.
\end{proof}

\section{Extension to General Step Sizes and Initialization}

\label{sec:general_decaying_step_size_analysis_final_corrected_appendix}
The main-text constant-step-size ergodic argument in Section~\ref{sec:ergodic_constant_proof_main} and the last-iterate argument in Section~\ref{sec:Non_Ergodic_Behavior_appendix} establish the base constant-step-size, zero-initialization guarantees for the STE-gradient method.

We now demonstrate that these convergence guarantees extend to a general class of non-summable step sizes under Assumption~\ref{assump:step_size} and to the possible nonzero initialization case (i.e., $\x^0 \neq 0$), thereby completing the proofs of Theorem~\ref{thm:Informal_ergodic_informal_nonasymp} and Theorem~\ref{thm:global_coordinate_recovery_noisy}.


Recall that the component-wise update rule is:
\begin{equation}
\label{eq:update_general_decaying_current}
    x_j^{t} = x_j^{t-1} - \eta_t \frac{\|\v\|^2}{2 \sqrt{2 \pi}} \left(\frac{1}{\sqrt{n}}\mathrm{sign}(x_j^{t-1}) - w_j^*\right) + \eta_t \epsilon^{t}_j, \qquad |\epsilon^{t}_j| \le \widetilde{\Delta}.
\end{equation}
The perturbation bound $\widetilde{\Delta}= {1 \over \tau} (C'\|\v\|_1^2 +  C''K_\xi\|\v\|_1) \sqrt{n/N}$ and the perturbation-to-drift strength ratio $\rho$
\[
\rho = \frac{\eta_t \cdot \widetilde{\Delta}}{\eta_t \cdot 2 \|\v\|^2 / (\tau \sqrt{n})} = \frac{ \widetilde{\Delta}}{ 2 \|\v\|^2 / (\tau \sqrt{n})} = \frac{(C'\|\v\|_1^2 +  C''K_\xi\|\v\|_1) n}{2\|\v\|^2 \sqrt{N}}
\]
are independent of $\eta_t$. We assume $N$ is large enough such that $\rho$ is sufficiently small (e.g., $\rho < 1/5$).

\subsection{Ergodic Convergence Analysis}

\begin{theorem}[Ergodic Convergence]
\label{thm:ergodic_general_step_size_final_corrected}
Suppose that Assumptions~\ref{assump:data}--\ref{assump:step_size} hold. Assume the sample size $N$ is large enough such
\[
N \ge \max\left( C_1 n^2, \left(  \frac{5n (C' \|\v\|_1^2 + C''K_\xi \|\v\|_1)}{2\|\v\|^2} \right)^2  \right)
\]
to ensure that the perturbation-to-drift strength ratio $\rho := \frac{n (C' \|\v\|_1^2 + C''K_\xi \|\v\|_1)}{2\|\v\|^2 \sqrt{N}} < 1/5$ for some positive universal constants $C_1, C'$ and $C''$.
Then, with probability at least $1 - 5m\exp(-cn)$ for some positive universal constant $c > 0$, the following statements hold.

The ergodic average $\overline{\w}^T = \frac{1}{T}\sum_{t=1}^T \w^t$ with $\x^0 = 0$ satisfies for a sufficiently large $T$:
\[
\left\| \overline{\w}^T - \w^* \right\|_\infty \le \frac{1}{\sqrt{n}} \cdot \frac{2 \rho}{1-1.1\rho} + O\left(\frac{1}{T\sqrt{n}}\right).
\]
Moreover, since the perturbation-to-drift strength ratio satisfies $\rho < 1/5$, then for $T$ sufficiently large, exact sign recovery is achieved by quantizing the ergodic average:
\[
\mathcal{Q}\!\left( \overline{\w}^T \right) = \w^*.
\]
\end{theorem}

\begin{proof}[Proof of Theorem~\ref{thm:ergodic_general_step_size_final_corrected}.]
The proof adapts the constant step size analysis by verifying the persistence of essential cycle properties under Assumption~\ref{assump:step_size}. We define $D_{\mathrm{eff}} := \frac{2\|\v\|^2}{\tau \sqrt{n}}$ as the constant factor for the full corrective drift magnitude.

\paragraph{Step 1: One-Step Reset Property.}
We show that for a sufficiently large $t$, if $x_j^t$ enters the incorrect sign regime, $x_j^{t+1}$ resets to the correct sign. Assume $w_j^* = +1/\sqrt{n}$. Let time $t$ be when $x_j^t < 0$ and $x_j^{t-1} \ge 0$.
The update for $x_j^t$ (using step size $\eta_t$ and based on $x_j^{t-1}$) involved zero drift (as $\mathrm{sign}(x_j^{t-1}) = \mathrm{sign}(w_j^*)$). Thus, $x_j^t = x_j^{t-1} + \eta_t \epsilon^{t}_j \ge -\eta_t \widetilde{\Delta}$.
The update for $x_j^{t+1}$ (using step size $\eta_{t+1}$) involves the full corrective drift $\eta_{t+1}D_{\mathrm{eff}}$:
\[ x_j^{t+1} = x_j^t + \eta_{t+1} D_{\mathrm{eff}} + \eta_{t+1} \epsilon^{t+1}_j. \]
Using $x_j^t \ge -\eta_t \widetilde{\Delta}$ and $\epsilon^{t+1}_j \ge -\widetilde{\Delta}$:
\[ x_j^{t+1} \ge -\eta_t \widetilde{\Delta} + \eta_{t+1} D_{\mathrm{eff}} - \eta_{t+1} \widetilde{\Delta}. \]
We substitute $\widetilde{\Delta} = \rho D_{\mathrm{eff}}$:
\[ x_j^{t+1} \ge -\eta_t (\rho D_{\mathrm{eff}}) + \eta_{t+1} D_{\mathrm{eff}} - \eta_{t+1} (\rho D_{\mathrm{eff}}), \] which gives
\[ x_j^{t+1} \ge D_{\mathrm{eff}} \left[ \eta_{t+1}(1 - \rho) - \eta_t \rho \right] = D_{\mathrm{eff}} \eta_t \left[ \frac{\eta_{t+1}}{\eta_t}(1 - \rho) - \rho \right]. \]
For $x_j^{t+1} > 0$, we need $\frac{\eta_{t+1}}{\eta_t}(1 - \rho) - \rho > 0$, or $\frac{\eta_{t+1}}{\eta_t} > \frac{\rho}{1 - \rho}$.
By Assumption~\ref{assump:step_size}, $\lim_{t\to\infty} \frac{\eta_{t+1}}{\eta_t} = 1$. Since $\rho < 1/5 < 1/2$, we have $\frac{\rho}{1 - \rho} < 1$. Thus, for a sufficiently large $t$, the condition holds, ensuring $x_j^{t+1} > 0$. A symmetric argument holds for $w_j^* = -1/\sqrt{n}$. This relies on Assumptions~\ref{assump:step_size} and $\rho < 1/5 < 1/2$.

\paragraph{Step 2: Cycle Structure and Occupation Time Bound.}
By symmetry, assume $w_j^* = +1/\sqrt{n}$ without loss of generality.

By Step 1 and Assumption~\ref{assump:step_size} ($\lim_{s\to\infty} \eta_{s+1}/\eta_s = 1$), there exists a finite $T_{\mathrm{burnin}}$ such that for all $t \ge T_{\mathrm{burnin}}$:
\begin{itemize}
    \item[(i)] The one-step reset condition $\eta_{t+1}/\eta_t > \rho/(1-\rho)$ holds since $0 < \rho <1/5$. In other words, if $x_j^t < 0$ (incorrect sign, having entered from $x_j^{t-1} \ge 0$), then $x_j^{t+1} \ge 0$ (one-step reset to the correct sign).
    \item[(ii)] The ratio of consecutive step sizes is bounded from above: $\eta_t/\eta_{t+1} \le M_0$ for some chosen universal constant $M_0 \ge 1$.
\end{itemize}

Since the initial phase for $t < T_{\mathrm{burnin}}$ is finite,  $x_j^{T_{\mathrm{burnin}}} \in [-R,R]$ for some finite $R > 0$. If $x_j^t$ becomes negative during $t < T_{\mathrm{burnin}}$, the corrective drift component in the update, $x_j^{s+1} \ge x_j^s + \eta_{s+1}D_{\mathrm{eff}}(1-\rho)$, is positive for each subsequent step $s$ while $x_j^s$ remains incorrect ($x_j^s < 0$). Due to the non-summability of $\{\eta_s\}$ (Assumption~\ref{assump:step_size}(2)), the sum of these positive drift contributions $\sum \eta_{s+1}D_{\mathrm{eff}}(1-\rho)$ must eventually exceed any finite negative value of $x_j^{t'}$, ensuring that $x_j^t$ returns to the correct sign regime in some finite time.

For this coordinate $j$, let $T'_{\mathrm{init},j}$ be a finite time index ($T'_{\mathrm{init},j} \ge T_{\mathrm{burnin}}$) by which $x_j^t$ has returned to the correct sign region and conditions (i) and (ii) above reliably hold for subsequent iterations. For the vector and $\ell_\infty$ analyses below, take the common burn-in time
\[
T'_{\mathrm{init}}:=\max_{j\in[n]}T'_{\mathrm{init},j},
\]
which is finite since $n$ is finite. After this common burn-in, the reset and cycle bounds hold simultaneously for all coordinates.

For $t \ge T'_{\mathrm{init}}$, the system exhibits a stable cyclic behavior:
\begin{itemize}
    \item An incorrect sign phase (e.g., $x_j^k < 0$ for $w_j^*=1/\sqrt{n}$, where $k \ge T'_{\mathrm{init}}$ and $x_j^{k-1} \ge 0$) lasts exactly one time step, with $x_j^{k+1} \ge 0$ due to the one-step reset property active for $k \ge T'_{\mathrm{init}} \ge T_{\mathrm{burnin}}$.

    \item We now determine the minimum duration, $L_{M_0}$, of the subsequent correct sign phase starting at $x_j^{k+1}$. After the reset at step $k+1$, the value is $x_j^{k+1} \ge D_{\mathrm{eff}}[\eta_{k+1}(1-\rho) - \eta_k\rho]$.
    During the correct phase (steps $s=k+1, k+2, \dots$), the update is $x_j^{s+1} = x_j^s + \eta_{s+1}\epsilon^{s+1}_j$. The maximum decrease in $x_j$ at any step $s+1$ within this phase is $\eta_{s+1}\widetilde{\Delta} = \eta_{s+1}\rho D_{\mathrm{eff}}$.
    By Assumption~\ref{assump:step_size}, $\eta_{s+1} \le \eta_{k+1}$ for all steps $s+1$ in this particular correct phase starting after $x_j^{k+1}$. Thus, the maximum decrease in any single step of this correct phase is bounded by $\eta_{k+1}\rho D_{\mathrm{eff}}$.

    Let $L_{M_0}$ be the minimum number of steps in this correct phase. For the phase to end (i.e., for $x_j$ to become negative again), the cumulative maximum decrease must overcome $D_{\mathrm{eff}}[\eta_{k+1}(1-\rho) - \eta_k\rho]$:
    \[ L_{M_0} \cdot (\eta_{k+1}\rho D_{\mathrm{eff}}) \ge D_{\mathrm{eff}}[\eta_{k+1}(1-\rho) - \eta_k\rho], \]
    or
    \[ L_{M_0} \ge \frac{\eta_{k+1}(1-\rho) - \eta_k\rho}{\eta_{k+1}\rho} = \frac{1-\rho}{\rho} - \frac{\eta_k}{\eta_{k+1}}. \]
    By Assumption~\ref{assump:step_size}(3), $\lim_{k\to\infty} \eta_k/\eta_{k+1} = 1$. Since $\eta_k \ge \eta_{k+1}$, we have $1 \le \eta_k/\eta_{k+1}$. We can choose $T_{\mathrm{burnin}}$ (and thus $T'_{\mathrm{init}}$) to be sufficiently large such that for all $k \ge T'_{\mathrm{init}}$, $\eta_k/\eta_{k+1} \le M_0$, where $M_0 \ge 1$ is a universal constant arbitrarily close to $1$. For a fixed, concrete non-asymptotic analysis for $k \ge T'_{\mathrm{init}}$ (where $T'_{\mathrm{init}}$ might be chosen to guarantee a specific $M_0$, e.g., $M_0=1.5$ if we assume $\eta_{k+1}/\eta_k \ge 2/3$), we have:
    \[ L_{M_0} \ge \frac{1-\rho}{\rho} - M_0 = \frac{1-(1+M_0)\rho}{\rho}. \]
    Since $L_{M_0}$ must be an integer, $L_{M_0} \ge \left\lceil \frac{1-(1+M_0)\rho}{\rho} \right\rceil$. Note that for a constant step size case, i.e., $\eta_k = \eta_{k+1}$, then $M_0$ can be set to $1$, so we recover $L \ge \left\lceil \frac{1-2\rho}{\rho} \right\rceil$, the cycle length relation and $\rho$ we used for the proof (Step 2) of Theorem~\ref{thm:ergodic_informal_nonasymp}.

\end{itemize}

Thus, for $t \ge T'_{\mathrm{init}}$, cycles consist of one incorrect step followed by at least $L_{M_0}$ correct steps (minimum length $1+L_{M_0}$). The behavior during an initial finite period (up to $T'_{\mathrm{init}}$) does not affect the asymptotic fraction of incorrect steps. Therefore, following the counting argument from Step 2 of the proof of Theorem~\ref{thm:ergodic_informal_nonasymp}, if $\gamma$ is the number of negative events in the interval $[T'_{\mathrm{init}}+1, T]$ and $T' = T - T'_{\mathrm{init}}$, we have for large $T'$ (i.e., large $T$):
\[
\gamma \le \frac{T'}{1+L_{M_0}} + 1 \le T' \frac{\rho}{1-M_0 \cdot \rho} + 1.
\]

\paragraph{Step 3: Bounding the Average Sign.}
From Step 2, we established that after an initial finite period $T'_{\mathrm{init}}-1$, the fraction of time of each component exhibiting an incorrect sign is bounded by $\frac{\rho}{1-M_0\rho}$ for the time in $[T'_{\mathrm{init}},T]$ for sufficiently large $T$. The initial $T'_{\mathrm{init}}-1$ steps contribute terms of order $O(T'_{\mathrm{init}}/T)$ to the final average sign error.

Given the asymptotic upper bound on the fraction of incorrect steps, the logic from Step 3 of the proof of Theorem~\ref{thm:ergodic_informal_nonasymp} applies. The same analysis in Step 3 of the proof of Theorem~\ref{thm:ergodic_informal_nonasymp}  yields the similar equation \eqref{eq:sum_bound}. Since the argument is purely algebraic and we have already presented it in Step 3 of the proof of Theorem~\ref{thm:ergodic_informal_nonasymp}, we omit the detailed re-computation. This computation yields
\[
\left| \frac{1}{T} \sum_{t=1}^T \mathrm{sign}(x_j^t) - w_j^* \sqrt{n} \right| \le \frac{2\rho}{1-M_0\rho} + O\left(\frac{T'_{\mathrm{init}}}{T}\right).
\]
Since $T'_{\mathrm{init}}$ is a finite constant ($\eta_{\max}$, $\rho$, $M_0$, but not $T$), the term $O(T'_{\mathrm{init}}/T)$ is $O(1/T)$.
Hence, the $\ell_\infty$ bound on $\overline{\w}^T$ becomes:
\[ \left\| \w^* -  \overline{\w}^T\right\|_\infty \le \frac{1}{\sqrt{n}} \left( \frac{2\rho}{1-M_0\rho} \right) + O\left(\frac{1}{T\sqrt{n}}\right). \]
Since $M_0 \ge 1$ can be chosen arbitrarily close to $1$, by choosing $M_0 = 1.1$ ($T'_{\mathrm{init}}$ depends on $M_0$ but still finite), the $L_\infty$ bound on $\overline{\w}^T$ and the exact recovery condition ($\rho < 1/5$, for large $T$ as detailed in the proof of Theorem~\ref{thm:ergodic_informal_nonasymp}) follows.

\end{proof}

The next corollary generalizes Theorem~\ref{thm:ergodic_general_step_size_final_corrected} to the case when $\x^0$ is not necessarily 0 but satisfies $\x^0 \in \left[-\frac{c_0}{\sqrt{n}} , \frac{c_0}{\sqrt{n}} \right]^n$.
Note that an initialization scale of $O(1/\sqrt{n})$ per component is consistent with standard practices such as Glorot normal initialization, which is commonly used for the initialization in training neural networks.

\begin{cor}[Ergodic Convergence with Bounded Initialization]
\label{cor:ergodic_general_nonzero_init}
Suppose the conditions of Theorem~\ref{thm:ergodic_general_step_size_final_corrected} hold, except that the initialization $\x^0$ is not necessarily $0$ but satisfies $\x^0 \in \left[-\frac{c_0}{\sqrt{n}} , \frac{c_0}{\sqrt{n}} \right]^n$ for some constant $c_0 > 0$.
Then, for a sufficiently large $T$ and with probability at least $1 - 5m\exp(-cn)$ for some universal constant $c > 0$, the conclusions of Theorem~\ref{thm:ergodic_general_step_size_final_corrected} still hold:
\[
\left\| \overline{\w}^T - \w^* \right\|_\infty \le \frac{1}{\sqrt{n}} \cdot \frac{2 \rho}{1-1.1 \rho} + \frac{C_{\mathrm{init}}}{T\sqrt{n}},
\] where $C_{init}$ is a finite constant that depends on the initialization bound constant $c_0, \|v\|$ and the step size $\eta_t$.
Moreover, if $T$ is sufficiently large, then $\mathcal{Q}\!\left( \overline{\w}^T \right) = \w^*$.
\end{cor}

\begin{proof}[Proof of Corollary~\ref{cor:ergodic_general_nonzero_init}.]
We demonstrate that a bounded non-zero initialization $\x^0$ does not alter the asymptotic conclusions of Theorem~\ref{thm:ergodic_general_step_size_final_corrected}. Without loss of generality, we analyze a single component $j$ and assume $w_j^* = 1/\sqrt{n}$. The other case follows by symmetry.

Case 1: $x_j^0 \ge 0$.
In this scenario, the iterate component $x_j^0$ starts with the correct sign. If $x_j^0 \ge 0$, the argument in the proof of Theorem~\ref{thm:ergodic_general_step_size_final_corrected} (specifically, $T_0$ is defined as the first time the iterate $x_j^t$ becomes negative as before) is unchanged, and the subsequent cycle analysis applies.

Case 2: $x_j^0 < 0$.
Let $T_{\mathrm{init}}$ be the first time $t \ge 1$ such that $x_j^t \ge 0$.
For $t \in [0, T_{\mathrm{init}}-1]$, $x_j^t < 0$, so $\mathrm{sign}(x_j^t) = -1$. The update rule for $x_j^{t+1}$ is:
\[ x_j^{t+1} = x_j^t - \eta_{t+1} \frac{\|\v\|^2}{\tau} \left(\frac{1}{\sqrt{n}}\mathrm{sign}(x_j^t) - w_j^*\right) + \eta_{t+1} \epsilon^{t+1}_j. \]
With $\mathrm{sign}(x_j^t) = -1$ and $w_j^* = 1/\sqrt{n}$, the drift term is $-\eta_{t+1} \frac{\|\v\|^2}{\tau} (-\frac{2}{\sqrt{n}}) = \eta_{t+1} D_{\mathrm{eff}}$, where $D_{\mathrm{eff}} = \frac{2\|\v\|^2}{\tau \sqrt{n}}$ is a constant.
The perturbation $\eta_{t+1}\epsilon^{t+1}_j$ satisfies $\eta_{t+1}\epsilon^{t+1}_j \ge -\eta_{t+1}\widetilde{\Delta}$. Using $\widetilde{\Delta} = \rho D_{\mathrm{eff}}$, this is $-\eta_{t+1}\rho D_{\mathrm{eff}}$, where $0< \rho < 1/5$ from the assumption in Theorem~\ref{thm:ergodic_general_step_size_final_corrected}.
Thus, for $t < T_{\mathrm{init}}$, we have
\[ x_j^{t+1} \ge x_j^t + \eta_{t+1} D_{\mathrm{eff}} - \eta_{t+1} \rho D_{\mathrm{eff}} = x_j^t + \eta_{t+1} D_{\mathrm{eff}} (1  - \rho ).  \]
Applying this repeatedly from $t=0$ to $T_{\mathrm{init}}-1$ yields
\[ x_j^{T_{\mathrm{init}}} \ge x_j^0 +  D_{\mathrm{eff}} (1  - \rho) \sum_{t=1}^{T_{\mathrm{init}}}\eta_{t}. \]
Since $x_j^0 \ge -c_0/\sqrt{n}$,
\[ x_j^{T_{\mathrm{init}}} \ge -\frac{c_0}{\sqrt{n}} +  D_{\mathrm{eff}} (1  - \rho) \sum_{t=1}^{T_{\mathrm{init}}}\eta_{t}. \]
Also, by definition of $T_{\mathrm{init}}$, $x_j^{T_{\mathrm{init}}-1} < 0$ and $x_j^{T_{\mathrm{init}}} \ge 0$. The update for $x_j^{T_{\mathrm{init}}}$ is $x_j^{T_{\mathrm{init}}} = x_j^{T_{\mathrm{init}}-1} + \eta_{T_{\mathrm{init}}} D_{\mathrm{eff}} + \eta_{T_{\mathrm{init}}} \epsilon^{T_{\mathrm{init}}}_{j}$.
Since $x_j^{T_{\mathrm{init}}-1} < 0$ and $\epsilon^{T_{\mathrm{init}}}_{j} \le \widetilde{\Delta} = \rho D_{\mathrm{eff}}$, we have:
\[ 0 \le x_j^{T_{\mathrm{init}}} < 0 + \eta_{T_{\mathrm{init}}} D_{\mathrm{eff}} + \eta_{T_{\mathrm{init}}} \rho D_{\mathrm{eff}} = \eta_{T_{\mathrm{init}}}D_{\mathrm{eff}} (1  + \rho ). \]
Combining these bounds on $x_j^{T_{\mathrm{init}}}$:
\[ D_{\mathrm{eff}} (1  - \rho) \sum_{t=1}^{T_{\mathrm{init}}}\eta_{t} \le x_j^{T_{\mathrm{init}}} - x_j^0 < \eta_{T_{\mathrm{init}}}D_{\mathrm{eff}} (1  + \rho ) + \frac{c_0}{\sqrt{n}}. \]
Therefore,
\begin{align*}
\sum_{t=1}^{T_{\mathrm{init}}}\eta_{t} &< \frac{\eta_{T_{\mathrm{init}}}D_{\mathrm{eff}} (1  + \rho ) + \frac{c_0}{\sqrt{n}}}{D_{\mathrm{eff}} (1  - \rho)} \\
&= \frac{\eta_{T_{\mathrm{init}}} (1  + \rho )}{ (1  - \rho)} + \frac{c_0/\sqrt{n}}{D_{\mathrm{eff}} (1  - \rho)} \\
&\le \frac{\eta_{\max} (1  + \rho )}{ (1  - \rho)} + \frac{\tau c_0}{2 \|\v\|^2 (1  - \rho)},
\end{align*}
where $\eta_{\max}$ is an upper bound on the step sizes (it can be set to $\eta_1$ since it is non-increasing).
Since $\sum_{t=1}^{\infty}\eta_{t} = \infty$ by Assumption~\ref{assump:step_size} and the right-hand side is a finite constant $T_{\mathrm{init}}$ must be finite, which depends on $c_0, \|v\|$ and the step size $\eta_t$.

After $T_{\mathrm{init}}$ iterations, $x_j^{T_{\mathrm{init}}} \ge 0$, and the subsequent analysis resembles the case where the iterate starts with the correct sign. The overall average $\frac{1}{T}\sum_{t=1}^T \mathrm{sign}(x_j^t)$ can be decomposed as below:
\[ \frac{1}{T}\sum_{t=1}^T \mathrm{sign}(x_j^t) = \frac{1}{T}\sum_{t=1}^{T_{\mathrm{init}}-1} \mathrm{sign}(x_j^t) + \frac{1}{T}\sum_{t=T_{\mathrm{init}}}^T \mathrm{sign}(x_j^t). \]
For $t \in [1, T_{\mathrm{init}}-1]$, $\mathrm{sign}(x_j^t) = -1$ (since $w_j^*=1/\sqrt{n}$ and $x_j^0<0$).
For the segment $[T_{\mathrm{init}}, T]$ of length $\widetilde{T} = T-T_{\mathrm{init}}+1$, the analysis from the proof of Theorem~\ref{thm:ergodic_general_step_size_final_corrected} (which adapts Theorem~\ref{thm:ergodic_informal_nonasymp}) implies:
\[ \left| \frac{1}{\widetilde{T}} \sum_{t=T_{\mathrm{init}}}^T \mathrm{sign}(x_j^t) - w_j^* \sqrt{n} \right| \le \frac{2\rho}{1-1.1\rho} + O\left(\frac{1}{\widetilde{T}}\right). \]
Thus, we have
\begin{align*}
\frac{1}{T}\sum_{t=1}^T \mathrm{sign}(x_j^t) - w_j^*\sqrt{n} &= \frac{1}{T}\sum_{t=1}^{T_{\mathrm{init}}-1} (\mathrm{sign}(x_j^t) - w_j^*\sqrt{n}) + \frac{1}{T}\sum_{t=T_{\mathrm{init}}}^T (\mathrm{sign}(x_j^t) - w_j^*\sqrt{n}) \\
&= \frac{T_{\mathrm{init}}-1}{T}(-1 - w_j^*\sqrt{n}) + \frac{\widetilde{T}}{T} \left( \frac{1}{\widetilde{T}} \sum_{t=T_{\mathrm{init}}}^T \mathrm{sign}(x_j^t) - w_j^*\sqrt{n} \right).
\end{align*}
Since $w_j^*=1/\sqrt{n}$, the first term is $\frac{T_{\mathrm{init}}-1}{T}(-2)$.
Thus, for large $T$,
\small
\[ \left| \frac{1}{T} \sum_{t=1}^T \mathrm{sign}(x_j^t) - w_j^*\sqrt{n} \right| \le \frac{2(T_{\mathrm{init}}-1)}{T} + \frac{T-T_{\mathrm{init}}+1}{T} \left( \frac{2 \rho}{1-1.1 \rho} + O\left(\frac{1}{T-T_{\mathrm{init}}+1}\right) \right) \]
\normalsize
\[ \le \frac{2 \rho}{1-1.1 \rho} + O\left(\frac{T_{\mathrm{init}}}{T}\right). \]

Therefore, the $\ell_\infty$ bound on $\overline{\w}^T$ stated in Theorem~\ref{thm:ergodic_general_step_size_final_corrected},
\[ \left\| \w^* - \overline{\w}^T \right\|_\infty \le \frac{1}{\sqrt{n}} \cdot \frac{2 \rho}{1-1.1 \rho}  + \frac{C_{\mathrm{init}}}{T\sqrt{n}}, \]
and the subsequent conclusions regarding exact recovery for sufficiently large $T$, hold even if $x_j^0$ starts within the specified bounded range, as the effect of the initial transient period $T_{\mathrm{init}}$ is absorbed into the $C_{\mathrm{init}}/T$ term of the bound.
\end{proof}

\subsection{Last-Iterate Convergence Analysis}

\begin{theorem}[Last-Iterate Convergence]
\label{thm:visits_general_decaying}
Suppose that Assumptions~\ref{assump:data}--\ref{assump:step_size} hold.
Assume that the sample size $N$ is large enough such that
\[
N \ge \max\left( C_1 n^2, \left(  \frac{5 (n+4)^2 (C' \|\v\|_1^2 + C''K_\xi \|\v\|_1)}{2\|\v\|^2} \right)^2  \right)
\]  for some positive universal constants $C_1, C'$ and $C''$.

Then, with probability at least $1 - 5m\exp(-cn)$ for some positive universal constant $c > 0$, the iterates $\w^t$ generated by the STE-gradient process with $\x^0 \in \left[-\frac{c_0}{\sqrt{n}}, \frac{c_0}{\sqrt{n}} \right]^n$ for a given $c_0 > 0$ visit the target state $\w^*$ infinitely often. Furthermore, assuming the noise variables $\xi^{(i)}$ are independent of the data and across samples, continuous, and not a.s. zero, then with $
\|v\|_0:=|\operatorname{supp}(\v)|=|\{j\in[m]:v_j\neq0\}|$,
with probability at least
\[
\frac{1}{2}- 2^{-\|v\|_0N-1}-5m\exp(-cn)
\]
over the draw of the dataset and the noise $\{(\Z^{(i)},\xi^{(i)})\}_{i=1}^N$, the iterates also depart from $\w^*$ infinitely often.

\end{theorem}

\begin{proof}[Proof of Theorem \ref{thm:visits_general_decaying}.]
We first note that the conclusions of this theorem regarding infinite visits and departures are robust to bounded non-zero initializations $\x^0 \in \left[-\frac{c_0}{\sqrt{n}} , \frac{c_0}{\sqrt{n}} \right]^n$. As shown in the analysis for Corollary~\ref{cor:ergodic_general_nonzero_init}, any initial phase where $x_j^0$ has an incorrect sign lasts for a finite time $T_{\mathrm{init}}$. The subsequent asymptotic analysis of recurrence and departure is unaffected by this finite transient period.

The proof adapts arguments from the constant step size case (Theorem~\ref{thm:noisy_recurrence_departure}) using properties established under Assumption~\ref{assump:step_size}.
First, the sample size requirement in Theorem~\ref{thm:visits_general_decaying} implies
that $\rho := \frac{(C'\|\v\|_1^2 +  C''K_\xi\|\v\|_1) n}{2\|\v\|^2 \sqrt{N}} < 1/(n+M_0)$ where $M_0 = 1.1$ by an algebraic computation. This in turn implies that $L_{M_0}+1 > n$, where $L_{M_0} = \lceil(1-(1+M_0)\rho)/\rho\rceil = \lceil(1-2.1\rho)/\rho\rceil$. Here $L_{M_0}$, rather than the constant-step-size quantity $L$, denotes the minimum correct-phase length in the general step-size setting; the use of this $L_{M_0}$ and $M_0 = 1.1$ is justified in Step 2 of the proof of Theorem~\ref{thm:ergodic_general_step_size_final_corrected}.

\textbf{Part 1: Infinite Visits to $\w^*$}

This relies on showing the fraction of time spent at $\w^*$ is positive. The cycle analysis presented in the proof of Theorem~\ref{thm:ergodic_general_step_size_final_corrected} establishes that, after a finite coordinate-dependent burn-in, each coordinate has cycles consisting of one incorrect step and at least $L_{M_0}$ correct steps. Taking the common burn-in time $T_{\mathrm{init}}^{\max}:=\max_{j\in[n]}T_{\mathrm{init},j}$, which is finite since $n$ is finite, makes this cycle structure hold simultaneously for all coordinates. Hence, up to the vanishing transient contribution $O(T_{\mathrm{init}}^{\max}/T)$, the fraction of time where $\w^t \neq \w^*$ is asymptotically bounded above by $n/(L_{M_0}+1)$. Since the condition $\rho < 1/(n+M_0)$ ensures $L_{M_0}+1 > n$, this fraction is strictly less than 1. Therefore, the fraction of time spent exactly at $\w^*$ is asymptotically bounded below by $c = 1 - n/(L_{M_0}+1) > 0$, making the rest of the argument in Part 1 of the proof of Theorem~\ref{thm:noisy_recurrence_departure} applicable.

\textbf{Part 2: Infinite Departures from $\w^*$}

This argument demonstrates that iterates cannot remain permanently at $\w^*$ on the event
\[
A:=\{\exists j\in[n]\text{ such that }w_j^*[\tilde{\nabla}_{\w}L(\w^*)]_j>0\}.
\]
By Lemma~\ref{lem:gradient_symmetry_nonvanishing}, the event $A$ occurs with probability at least $
\frac{1-2^{-\|v\|_0N}}{2}$.

Assume we are in such a realization for some coordinate $j$. Suppose, for contradiction, that $\w^u=\w^*$ for all $u\ge t'$. Then $
x_j^{k}=x_j^{k-1}-\eta_k[\tilde{\nabla}_{\w}L(\w^*)]_j$.
Summing from $k=t'+1$ to $T+1$ gives
\[
x_j^{T+1}=x_j^{t'}-[\tilde{\nabla}_{\w}L(\w^*)]_j\sum_{k=t'+1}^{T+1}\eta_k.
\]
Since $[\tilde{\nabla}_{\w}L(\w^*)]_j$ has the same sign as $w_j^*$ on $A$ and $\sum_k\eta_k=\infty$ by Assumption~\ref{assump:step_size}, the coordinate $x_j^k$ must eventually cross the quantization boundary in the direction opposite to $w_j^*$. This contradicts $\w^u=\w^*$ for all $u\ge t'$.

Combining the event $A$ with the finite-sample approximation event from Theorem~\ref{thm:RAIC_for_two_layer_quantized_neural_networks_appendix}, which holds with probability at least $1-5m\exp(-cn)$, gives the departure probability lower bound
\[
\frac{1}{2}- 2^{-\|v\|_0N-1}-5m\exp(-cn)
\]
Since the visits to $\w^*$ are infinite, departures from $\w^*$ are also infinite on this event.

\end{proof}

\end{document}